\documentclass[preprint,3p,times]{elsarticle}
\usepackage{amsmath,amssymb,amsfonts, amsthm}

\usepackage{subcaption}
\usepackage{dirtytalk}
\usepackage{textcomp}
\usepackage{xcolor}

\usepackage{booktabs}
\usepackage{algpseudocode}
\usepackage{multirow}

\usepackage{graphicx}

\usepackage[utf8]{inputenc}
\usepackage[english]{babel}
\usepackage{makecell}
\usepackage[justification=centering]{caption}

\newtheorem{theorem}{Theorem}

\newtheorem{corollary}{Corollary}
\newtheorem{lemma}{Lemma}
\newtheorem{definition}{Definition}

\usepackage{tikz}

\usepackage{verbatim}

\usepackage[linesnumbered,ruled,vlined]{algorithm2e}
\SetKwInput{KwInput}{Input}        % Set the Input
\SetKwInput{KwOutput}{Output}       % set the Output
\usepackage[detect-all]{siunitx}
\usepackage{xcolor}

\journal{Neural Networks}

\begin{document}

\begin{frontmatter}

\title{Generating Post-hoc Explanations for Skip-gram-based Node Embeddings by Identifying Important Nodes with \textit{Bridgeness}}

\author[label1]{Hogun Park\corref{cor1}%
\fnref{fn1}}
\ead{hogunpark@skku.edu}

\author[label2]{Jennifer Neville}
\ead{neville@cs.purdue.edu}

\cortext[cor1]{Corresponding author.}
\address[label1]{Sungkyunkwan University, Republic of Korea}
\address[label2]{Purdue University and Microsoft Research, USA}

%\markboth{Journal of \LaTeX\ Class Files,~Vol.~18, No.~9, November~2021}%
%{How to Use the IEEEtran \LaTeX \ Templates}

%\maketitle

\begin{abstract}
Node representation learning in a network is an important machine learning technique for encoding relational information in a continuous vector space while preserving the inherent properties and structures of the network. Recently, \textit{unsupervised} node embedding methods such as DeepWalk \citep{deepwalk}, LINE \citep{line}, struc2vec \citep{struc2vec}, PTE \citep{pte}, UserItem2vec \citep{wu2020multi}, and RWJBG \citep{li2021random} have emerged from the Skip-gram model \citep{word2vec} and perform better performance in several downstream tasks such as node classification and link prediction than the existing relational models. However, providing post-hoc explanations of Skip-gram-based embeddings remains a challenging problem because of the lack of explanation methods and theoretical studies applicable for embeddings. In this paper, we first show that global explanations to the Skip-gram-based embeddings can be found by computing \textit{bridgeness} under a spectral cluster-aware local perturbation. Moreover, a novel gradient-based explanation method, which we call GRAPH-wGD, is proposed that allows the top-$q$ global explanations about learned graph embedding vectors more efficiently. Experiments show that the ranking of nodes by scores using GRAPH-wGD is highly correlated with true \textit{bridgeness} scores. We also observe that the top-$q$ node-level explanations selected by GRAPH-wGD have higher importance scores and produce more changes in class label prediction when perturbed, compared with the nodes selected by recent alternatives, using five real-world graphs
\end{abstract}


\end{frontmatter}

%!TEX root = main.tex

\section{Introduction}

Node embedding has garnered considerable research interest in recent years. The fundamental problem is to find a low-dimensional representation while preserving the inherent relational properties in an embedding space (i.e., if neighbors of two vertices are similar in the original network, they should have similar embedding vectors). Earlier works have demonstrated that in addition to pairwise edges, high-order proximities and structural similarities between nodes are important in capturing the underlying structure of the network \citep{deepwalk,line}. In particular, the Skip-gram model \citep{mikolov2013distributed} in natural language processing (NLP) has been actively extended to several node embedding methods such as DeepWalk \citep{deepwalk}, node2vec \citep{node2vec}, struc2vec \citep{struc2vec}, and RWJBG \citep{li2021random}. Although the Skip-gram-based node embedding methods have increased the performance of various downstream tasks such as node classification and link predictions, it is still difficult to interpret the model results.

\textcolor{black}{The issue of interpretability has been a concern among researchers, leading to the development of new methods aimed at providing explanations based on salient features \citep{lime,lundberg2017unified,ribeiro2018anchors}, increasing trust in AI systems \citep{calisto2022breastscreening,calisto2022modeling,calisto2021introduction,calisto2020breastscreening,calisto2017towards}, or making the models more transparent \citep{Chu2018}. For instance, in the medical field, heatmap-based explanations for BreastScreening-AI systems \cite{calisto2022breastscreening,calisto2022modeling} have been shown to increase clinician trust in the AI system. The explanations highlight the circularity of masses or the distribution of microcalcifications using different colors in X-ray or MRI images.}
However, the most recent methods have focused on providing explanations for models on simple sequenced or grid inputs. Recently, the explainability of graph neural networks (GNNs) \citep{gcn,wu2020comprehensive,wu2022semi,gnnnn1,gnnnn2,gnnnn3} has been explored \citep{gnnexplainer, Luo2020, pgm, subgraphx, huang2022graphlime,luo2020parameterized,lucic2022cf}, but these studies are limited to understanding supervised GNNs. For example, GNNExplainer \citep{gnnexplainer} was developed to provide local explanations for GNNs that were trained for node-classification tasks. They suggested the notion of importance by leveraging the similarity between the new prediction under a candidate subgraph (i.e., explanation) and the original prediction using mutual information. Therefore, it is not directly applicable to {\em unsupervised} node embedding models such as DeepWalk and LINE. In this paper, we propose a novel explanation method for finding globally important nodes using learned Skip-gram-based network embeddings. We formalize a node-level global explanation in embeddings and show that node-level explanations of several Skip-gram-based embedding models such as LINE, PTE, and struc2vec, are related to \textit{bridgeness} under spectral cluster-aware local perturbation. However, capturing of clusters is computationally inefficient, and a novel gradient-based explanation method, which we call GRAPH-wGD, is proposed to find the top-$q$ global explanations of learned graph embeddings.

\begin{figure*}[!t]
     \centering
     \includegraphics[width=0.99\textwidth]{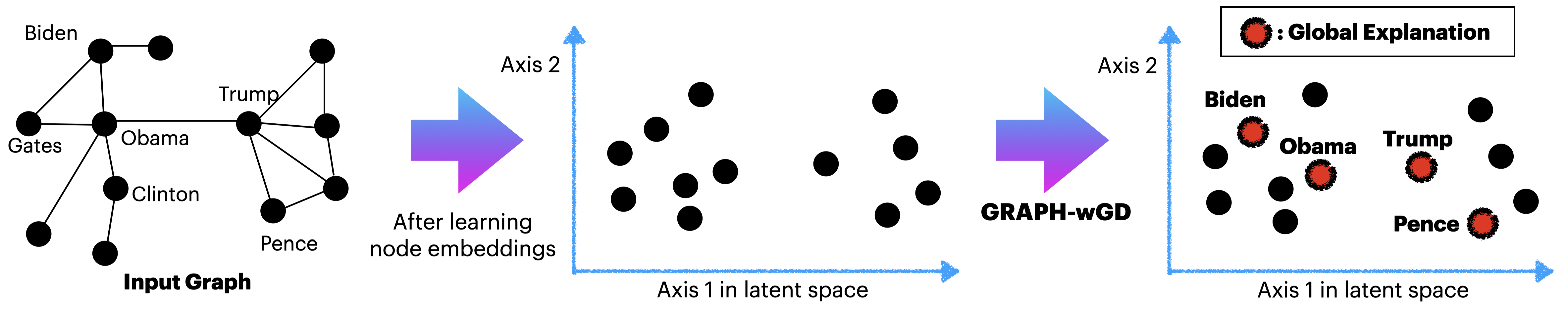}
     \caption{Illustration of GRAPH-wGD. GRAPH-wGD provides a global explanation for learned DeepWalk embeddings. (Best Viewed in Color.)}
	 \label{fig:obama:vis}
\end{figure*}

For example, in Figure \ref{fig:obama:vis}, an input social graph and its DeepWalk embeddings are visualized in the left two sub-figures. GRAPH-wGD can provide a global explanation to the learned DeepWalk embeddings. It generates an explanation by identifying a set of important nodes that are most influential to the embedding. In the example social graph, DeepWalk learns two social parties around Obama and Trump, and GRAPH-wGD provides four nodes that are the most influential to the learned embeddings. The explanation makes the learned embedding more interpretable and provides useful clues about whether the learned embeddings are trustable or not. For several other time-critical and cost-intensive node embedding-based applications in biology \citep{nelson2019embed} and medicine \citep{zitnik2019machine}, providing interpretable explanations can also be potentially valuable.

In the experiment, we evaluate GRAPH-wGD over five different real-world graphs and compare it with several alternatives. The results show that the node-level explanations identified by GRAPH-wGD are correlated with \textit{bridgeness} and that the top identified nodes have greater effects on neighbors and produce larger changes in prediction probability when they are perturbed.

\paragraph{Problem Definition}
Let $W \in \mathbb{R}^{|V| \times k}$ be a fixed $k$-dimensional node embedding computed for a graph $G$ with $V$ vertices and $E$ edges using model $\mathcal{M}$, with parameters $\theta_{\mathcal{M}}$ and loss function $\mathcal{L}_{\mathcal{M}}$. Our goal is to {\em score} each $v_j \in V$ with respect to (w.r.t.) their {\em importance} in $W$ and identify the nodes having the greatest global effect on the learned embedding $W$\footnote{Note that although some embedding methods are unstable (i.e., node locations oscillate during learning), our goal is to score the nodes conditioned on their locations in the final embedding.} as \textit{post-hoc node-level global explanations}.

\paragraph{Organization}
First, previous approaches are discussed in Section \ref{sec:relatedwork}. Section \ref{sec:bridgeness} provides some preliminaries, including Skip-gram-based node embedding models and our node-level explanation method. We found that the node-level explanations were theoretically related to bridge nodes that connect across different clusters in the embedding space. To efficiently capture the nodes that have high \textit{bridgeness} scores (i.e., to obtain the top-$q$ \textit{node-level global explanations}), a gradient-based explanation model, GRAPH-wGD, is proposed in Section \ref{sec:graphgd_wgd}, and its theoretical relationship to \textit{bridgeness} is described in Section \ref{sec:ta}. Our experimental evaluation is reported in Section \ref{sec:exp}. Finally, Section \ref{sec:conclusion} concludes our paper.

%!TEX root = main.tex

\newpage
\section{Related Work}\label{sec:relatedwork}

\subsection{Skip-gram-based Node Embedding Methods}

\textcolor{black}{
In the field of natural language processing (NLP), the Skip-gram model~\citep{word2vec} has been proposed for generating low-dimensional embeddings for text data efficiently. This architecture arranges semantically similar words close together, and has been effectively used in many NLP tasks. The Skip-gram architecture was later applied to learning node representations by DeepWalk~\citep{deepwalk}. It uses random walks to generate input sequences for learning and obtains low-dimensional node representations in a graph. Nodes with similar random walks tend to have similar positions and similar embeddings. This idea was further developed by node2vec~\citep{node2vec} which uses second-order random walks and outperforms existing baselines in node classification and link prediction tasks.
Furthermore, the cACP-DeepGram method~\citep{cACPDeepGram} utilizes the FastText-based word embedding technique to create a low-dimensional representation for each peptide. This representation is then utilized as a descriptor by a deep neural network model to accurately differentiate between ACPs. FastText extends the traditional Skip-gram model by incorporating $n$-gram features while retaining the original optimization framework. Despite the popularity of Skip-gram based node embedding models in various downstream tasks, it remains a challenge to interpret the learned embeddings.}

\textcolor{black}{
Recently, researchers have extensively analyzed the Skip-gram architecture to better understand it. GMF-FE~\citep{zhu2021free} generalizes the Skip-gram learning method by using a free energy distance function and matrix factorization. The method demonstrates that Skip-gram embeddings can be interpreted as an implicit similarity matrix that is learned with sampling under free energy distance theory and with additional scaling parameters. Despite GMF-FE's effort to map the Skip-gram learning to a matrix factorization problem, it is still challenging to understand which nodes contribute to the Skip-gram-based embeddings. In particular, the node representations in the latent factor space are not useful for providing explanations for the recommendations. It is also important to note that the matrix factorization optimization problem is non-convex~\citep{koren2009matrix}, meaning that the most similar nodes to a given node may not necessarily be the neighbors in the latent space. 
Similarly, GANRL~\citep{wu2021unified} aims to propose a unified generative adversarial learning framework for learning Skip-gram network representation methods. This approach improves the performance of Skip-gram network representation learning by obtaining more appropriate proximity scores through adversarial learning methods. However, it is still not possible to estimate which nodes contribute most to the learning. Recently, ASK~\citep{hsieh2022toward} has been proposed as a solution to the issue of needing good random samples for learning the Skip-gram model. Instead, it uses the $k$-most important neighbors to select positive samples and leverages their personalized PageRank (PPR) scores to update the original objective functions. The method has been shown to perform better in link prediction and node classification tasks without the need for expensive negative sampling, but it is still limited in terms of providing post-hoc explanations for the learned embeddings. Although the above methods have focused on improving the performance and understanding of the Skip-gram architecture, the ability to interpret the learned embeddings is still limited.}

\subsection{\textcolor{black}{Explainable Artificial Intelligence (XAI) Methods for Graphs}}

Previous approaches such as those of \citep{Fong2017} and \citep{Li2016} directly evaluated the importance of the discovered features by measuring the difference in prediction scores after perturbation of the features. Formally, the importance function $g$ can be formulated as 
      % \vspace{-4mm}
      \begin{equation}
          g(q_j | \textsc{x}, c) = p(\hat{y}=c | \textsc{x})  - p(\hat{y}=c | \textsc{x}_{-q_j}),
          \label{equ:greadyalg}
      \end{equation}
      % \vspace{-4mm}
            
\noindent where $p(\hat{y}=c | \textsc{x})$ is a predicted probability from a learned model $\mathcal{M}$, given that the input feature $\textsc{x}$ for a class label $c$ and $p(\hat{y}=c | \textsc{x}_{-q_j})$ is the new prediction after the perturbation of a partial feature $q_j$. We call this approach a \textit{Greedy} method for comparison in our experiments. However, some predictions (e.g., softmax outputs) are often not calibrated well \citep{guo2017calibration} and are not applicable to representation learning methods such as node-embedding models.

Recently, more-accurate model-agnostic methods have also been proposed to develop feature importance scores to identify locally important features for the given input instance. For example, in LIME \citep{lime} an interpretable ML model was proposed that could identify locally important features by a local perturbation approach. However, it is not applicable to graph-based node-embedding models, which are unsupervised learning models. To understand predictions of the neural networks, several recent works (e.g., \citealp{simonyan2013deep,sundararajan2017axiomatic}) have leveraged gradients to find feature attributions (or saliency maps). However, they are also inapplicable to node-embedding models (e.g., DeepWalk and LINE) owing to their sampling-based optimization, and the backpropagation-based methods are not tractable to process all complex combinations of nodes/paths.

%\subsection{Interpretation on Graph-based ML models}

    Degree and PageRank are basic network statistics to obtain the graph-wise importance of global nodes, but they are not post-hoc methods. In particular, although random walk-based node embeddings such as DeepWalk, heavily rely on high-degree nodes, high-degree nodes are not guaranteed to have the highest global effect on the learned model. For example, in the worst case, if the high-degree nodes are placed in a single cluster, the perturbation on the nodes may not considerably change label prediction or clustering in the embedding space. Regarding \textit{post-hoc explanations} for graph-based ML models, recent explanation models for GNNs~\citep{gcn,wu2020comprehensive,wu2022semi,gnnnn1,gnnnn2,gnnnn3} using gradient \citep{Pope2019}, decomposition \citep{Schnake2020}, surrogate \citep{graphlime,pgm}), and perturbation \citep{gnnexplainer,Luo2020,Schlichtkrull2020} have been proposed. However, all these approaches assume that their target models are supervised GNN models and are not applicable to unsupervised node-embedding models. For example, GNNExplainer \citep{gnnexplainer}, SubgraphX \citep{subgraphx}, CP-Explainer \citep{park2022providing}, and GraphMask \citep{Schlichtkrull2020} leverage mutual information, Shapley values, conformal prediction, and divergence, respectively, using both the original and affected label predictions, which are predicted on the candidate explanation subgraph. 
    They cannot provide global explanations for the node embedding. Moreover, \citep{huang2022graphlime,luo2020parameterized,pgm,lucic2022cf} have recently proposed generating explanations for supervised GNN models, but these methods cannot be applied to unsupervised node embedding. The taxonomy induction method \citep{Liu2018} and XGNN \citep{xgnn} are also relevant to this work. However, the taxonomy induction method interprets node embeddings only by identifying clusters in the embedding space and XGNN provides a model-level interpretation with graph generation, which still assumes that its target model is a supervised model. Thus, they do not quantify the global effect of each node for unsupervised node-embedding models. Meanwhile, a supervised importance score regression model, GENI \citep{geni}, has also been proposed. Although our method requires a set of scored node sets, it does not need any importance labels.

%!TEX root = main.tex

\vspace{-1mm}
\section{Explanation and the Effect of High Bridgeness in Node Embeddings}\label{sec:bridgeness}

\subsection{Skip-gram-based Node Embedding}

Several node-embedding problems including Skip-gram-based embedding models can be formulated as learning a function $f_{\mathcal{M}}: V \rightarrow \mathbb{R}^k$, which is a mapping function from nodes to feature representations given a graph $G = (V, E)$. Here, $\mathcal{M}$ denotes an embedding model and $k$ is a parameter specifying the number of feature representation dimensions. Equivalently, $W$ is a matrix form of $f_{\mathcal{M}}$ and of size $|V| \times k$ parameters. For every source node $v_b \in V$, we define $\mathcal{N}_G(v_b) \subset V$ as a network neighborhood of node $v_b$. For example, DeepWalk \citep{deepwalk} optimizes the following objective function, which maximizes the log-probability of observing a network neighborhood $\mathcal{N}_G(v_b)$ for a node $v_b$:
%\vspace{-2mm}
\begin{equation}
    \resizebox{0.34\hsize}{!}{%    
        $\text{arg}\!\,\max\limits_{f_{\mathcal{M}}} \sum_{v_b \in V} \log Pr(\mathcal{N}_G(v_b) | f_{\mathcal{M}}(v_b)).$
    }
\end{equation}

\label{eq:deepwalkloss}

\noindent For learning the abovementioned function, pairwise learnings such as hierarchical softmax or negative sampling \citep{word2vec}  using a sampled set of $\{(v_b ,v_i), v_i\in \mathcal{N}_G(v_b))\}$ for $\forall v_b \in V$ are often chosen.
\vspace{-2mm}

\subsection{Explanation for Node Embedding} \label{subsec:nodeimpdef}
To measure the effect of each node $v_b \in V$ in any learned latent representation $W$, we first define \textit{node-level explanation}.

\begin{definition}[Node-level explanation] Let $G$ be an input graph with vertices $V$ and $G'_{-v_b}$ be a perturbed version of $G$ after rewiring edges of $v_b$. Here, $W \in \mathbb{R}^{|V| \times k}$ and ${W'}_{-v_b} \in \mathbb{R}^{|V| \times k}$ are latent representations over $G$ and ${G'}_{-v_b}$, respectively, and 
$\mathcal{N}_W(v_b|m)$  returns the $m$-nearest neighbors of $v_b$ in the latent space $W$. 
The \textbf{node-level explanation} $v_b$ for the learned $W$ can be found as follows:

\vspace{-2mm}
\begin{equation}
    \resizebox{0.54\hsize}{!}{%    
        $\text{arg}\!\,\max\limits_{v_b \in V} Imp(v_b| W, W'_{-v_b}) =  \frac{1}{|V|} \sum_{v_i \in V} \left( 1-\frac{|\mathcal{N}_W(v_i|m) \cap \mathcal{N}_{W'_{-v_b}}(v_i|m)|}{\left| \mathcal{N}_W(v_i|m)\right|} \right)$.
     }
\label{equ:newnodeimpscore}    
\end{equation}

\label{def:newnodeimp}
\end{definition}
\vspace{-2mm}

\noindent  \textcolor{black}{The node importance metric to the embedding, referred to as \textit{Imp}, is calculated as the average proportion of neighbors that change in the embedding when the edges of node $v_b$ are perturbed. Although it is a useful measure for several unsupervised learning tasks that rely on $m$-nearest neighbors, such as local clustering, link prediction, and classification, it has some limitations. One of these limitations is the difficulty in generalizing the metric due to the presence of the hyper-parameter $k$.  Additionally, the embeddings can be affected by spurious information from nodes with different degrees, and it is necessary to mitigate this noise with minimal deviation techniques such as normalizations~\citep{dunn1974well,koren2003spectral}. To address these challenges, we propose a degree-normalized node-level explanation that measures the changes in pairwise distances (PD) across all nodes.}
\vspace{-1mm}

\begin{definition}
[(\textcolor{black}{Degree-normalized}) Node-level Explanation] Let $G$ be an input graph with vertices $V$ and $G'_{-v_b}$ be a perturbed version of $G$ after rewiring the edges of $v_b$. Here, $d_b$ represents the degree of node $v_b \in V$. Let $W \in \mathbb{R}^{|V| \times k}$ and ${W'}_{-v_b} \in \mathbb{R}^{|V| \times k}$ be latent representations learned from $G$ and ${G'}_{-v_b}$, respectively.

\vspace{3mm}
\resizebox{0.98\hsize}{!}{%  
\begin{minipage}{\linewidth}
    \begin{equation}
        \begin{aligned}
            &\text{arg}\!\,\max\limits_{v_b \in V} \widehat{\text{Imp}}\left(v_b|G, {G'}_{-v_b}, W,{W'}_{-v_b}\right) = ||PD({G, W}) - PD(G'_{-v_b}, W'_{-v_b})||\\
            &\;\;\;= \left|\left|\sum_{v_i, v_j \in V} {\left|\left| \frac{\vec{w}_{i}}{\sqrt{d_i}} - \frac{\vec{w}_{j}}{\sqrt{d_j}} \right|\right|^2}     - \sum_{v_i, v_j \in V}{\left|\left| \frac{\vec{w}'_{i}}{\sqrt{d_i}} - \frac{\vec{w}'_{j}}{\sqrt{d_j}}\right|\right|^2}\right|\right| \\
        \end{aligned}
    \end{equation}
\end{minipage}    
}
\label{def:pairwisenodeimp}
\end{definition}
% \vspace{-4mm}

\noindent This version of \textit{Imp} above measures how the sum of pairwise distances is changed after perturbation of $v_b$. Compared to Def. \ref{def:newnodeimp}, it is more computationally expensive to compute ($O({|V|}^2$) but helps to have better theoretical understanding. We note that, for identifying globally important nodes w.r.t. $W$, we measure importance scores for all nodes and return the nodes with the highest score. 

To calculate $\widehat{Imp}$, we define a degree/volume-preserving perturbation method by using Def. \ref{def:perturb} below.

%\vspace{-mm}
\begin{definition}[Degree/volume-preserving cluster-aware perturbation] Let $G$ be an input graph with vertices $V$, and $A$ be the adjacency matrix of $G$. Given a graph $G$ and its set of (arbitrary) clusters $C=\{C_1, ..., C_k\}$, we define a function {\em Perturb}($G$, $v_b$, $\alpha$) = $G'_{-v_b}$ to perturb inter-cluster edges around $v_b \in V$ with the perturbation ratio $\alpha$,  as with the following steps.:
\vspace{-2mm}
    \begin{enumerate}
                \item {\small Set $V_{cand} =\{v_i | v_i \in C_i, A[b,i] >0 \}$;}
                \item {\small Sample $Q$, which is composed of $\alpha \cdot |V_{cand}|$ vertices, from $V_{cand}$;}
                \item {\small For all $v_q \in Q$, // Remove inter-cluster edge $(v_b,v_q)$ and $(v_q, v_b)$, adjust by adding self-edges} \\
                    {\small $A[q,q] = A[q,q] + A[b,q]$, $A[b,b] = A[b,b] + A[b,q]$}\\            
                    {\small $A[b,q] = 0$, $A[q,b] = 0$}
                \item {\small Return $G'_{-v_b}=A$.}
    \end{enumerate}
    \vspace{-2mm}
\label{def:perturb}    
\end{definition}

\noindent Because the graph $G$ is undirected, by having $A[b,q]$=0, we can also set $A[q,b]=0$ and the additions to $A[q,q]$ and $A[b,b]$ do not change the degrees of $v_q$ and $v_b$. Although similar degree-preserving perturbations have been proposed by \citep{zitnik2018prioritizing}, they are global perturbation approaches, which take $O(|E|)$ time complexity, and the overall cluster memberships are likely to be changed. By using the strategy in definition \ref{def:perturb}, we have a new graph $G'_{-v_b}$, which still has the same degrees $D$ and volume of $G$. We note that our \textit{Perturb}($G$, $v_b$, $\alpha$) corresponds to $O(|d_{max}|)$, where $d_{max}$ is the maximum node degree.

\vspace{-2mm}

\subsection{Bridgeness and Node-level Explanation for Spectral Embeddings}

We initially show the relationship between node-level explanation (as in definition \ref{def:pairwisenodeimp}) and the latent representation returned from spectral embedding. First, we include the formal definition of a \textit{bridge node} \citep{wang2011identifying,jensen2015detecting}. In agreement with earlier findings in \citep{10.1371/journal.pone.0012528,batada2006stratus}, we assume that bridge nodes have more inter-modular positions than community cores. The existence of bridge nodes often leads to more inter-cluster edges. The {\em bridge node} connects the clusters and can enhance the integration of the whole network.

\begin{definition}[Bridge node and bridgeness \citep{ghalmane2019immunization}] Given an arbitrary cluster set $C=\{C_1, ..., C_k\}$ in $G$, when a node $v_b \in C_b$ connects at least two different clusters including $C_b$, it is called a \textit{bridge node}. The \textbf{bridgeness} of $v_b$ is as follows:
	\vspace{2mm}
    \begin{equation}
        Bridgeness(v_b \in C_b| C, G) = \gamma (C_b) \cdot d^{inter}_{b},
        \label{equ:bridgeness}
    \end{equation}
%    \vspace{-5mm}  
    \label{def:bridgenode:bridgeness}  
\noindent where $d^{inter}_{b}$ is the inter-cluster degree of node $v_b \in C_b$, which is computed from $G$ and $C$. Here, $\gamma(C_b)$ represents the importance of clusters $C_b$. For $\gamma (C_b)$, the number of neighboring clusters to $C_b$ is used in \citep{ghalmane2019immunization}, and the distance to each cluster is used \citep{jensen2015detecting,rao1982diversity}. 
In this paper, we set $\gamma(C_b)=1$ for simplicity.
\end{definition}

Now we can derive which node is a \textit{node-level explanation} w.r.t. a spectral embedding (eigenvector $U$ from $D^{-1/2}AD^{-1/2}$), where $D$ is the degree matrix of $A$.

\begin{theorem}
Let $\{C_1, ..., C_k\}$  be $k$ clusters from spectral clustering of graph $G$, with nodes $V$. Let $G'_{-v_j}$ be a perturbed version of $G$ using Perturb($G$, $v_j$, $\alpha\!=\!1$) (definition \ref{def:perturb}). Let  $U \in \mathbb{R}^{|V| \times k}$ and $U'_{-v_j} \in \mathbb{R}^{|V| \times k}$ be the eigenvectors of $G$ and ${G'}_{-v_j}$, respectively. Then, \\
%\vspace{-2mm}
\begin{equation*}
\resizebox{0.58\hsize}{!}{%
    $\text{arg}\!\,\max\limits_{v_j \in V} \widehat{Imp}(v_j|G, {G'}_{-v_j}, U,{U'}_{-v_j}) = \text{arg}\!\,\max\limits_{v_j \in V} Bridgeness(v_j|C, G)$.
}
\end{equation*}
\label{theorem:eigenvec:bridge}
\end{theorem}
\vspace{-5mm}

For proofs of lemmas and theorems, see \ref{sec:proofs}.

\vspace{-2mm}

\subsection{Connection to Skip-gram Embedding Methods}
Next, we show which node is a \textit{node-level explanation} for the DeepWalk embedding method.

\begin{theorem}
Let $G'_{-v_j}$ represent a perturbed version of graph $G$ with $V$ nodes, using Perturb($G$, $v_j$, $\alpha\!=\!1$) (definition \ref{def:perturb}). Let $W \in \mathbb{R}^{|V| \times k}$ and ${W'}_{-v_j} \in \mathbb{R}^{|V| \times k}$ be embeddings from DeepWalk (with context window size $r$) over $G$ and ${G'}_{-v_j}$, respectively. 
Let $\tilde{C}=\{\tilde{C}_1, ..., \tilde{C}_k\}$ be $k$ clusters from spectral clustering of $\tilde{G}_r$, where $\tilde{G}_r$ is an $r^{th}$ weighted power transformation of $G$ 
that is degree-normalized. 
Then,\\
%\vspace{-2mm}
\begin{equation*}
\resizebox{0.6\hsize}{!}{%
    $\text{arg}\!\,\max\limits_{v_j \in V} \widehat{Imp}(v_j|G, {G'}_{-v_j}, W,{W'}_{-v_j}) = \text{arg}\!\,\max\limits_{v_j \in V} Bridgeness(v_j| \tilde{C}, \tilde{G}_r)$.
}
\end{equation*}

\vspace{-2mm}
\label{theorem:n2v:bridge}
\end{theorem}

\noindent This theorem shows that the node with the highest \textit{bridgeness} in $\tilde{G}_r$ is also the node with the largest importance score (thus, node-level explanation for $W$) for the DeepWalk embedding over $G$. In the Appendix, we outline how the same node is also likely to have the highest bridgeness score in $G$ (i.e.,  $\text{arg}\!\,\max\limits_{v_j \in V} Bridgeness(v_j| \tilde{C}, \tilde{G}_r)  \approxeq  \text{arg}\!\,\max\limits_{v_j \in V} Bridgeness(v_j| C, G)$). 

Other embedding models such as LINE, PTE, and struc2vec can be analyzed in the same way. For LINE and PTE, similar theoretical approximations were studied in \citep{qiu2018network}. As in Table 1 in \citep{qiu2018network}, we can expect the same theoretical properties w.r.t.\ perturbation for both LINE \citep{line} and PTE \citep{pte} because our perturbation method does not change degree matrices and other constants. struc2vec \citep{struc2vec} learns low-dimensional representations using a Skip-gram architecture on the context graph, which is constructed by structural similarity among nodes. When input adjacency is assumed as the context graph, its bridge nodes are still important as in DeepWalk. Similarly, RWJBG \citep{li2021random} learns node embeddings on joint behavior graphs under the same Skip-gram framework and our analysis remains applicable. UserItem2vec \citep{wu2020multi} also constructs an attributed heterogeneous user/item network and the embeddings are learned on the Skip-gram architecture; the most important nodes are identifiable under our proposed idea.

\begin{figure*}[t]
 \centering
 \begin{subfigure}[t]{0.49\textwidth}
  \includegraphics[width=\textwidth]{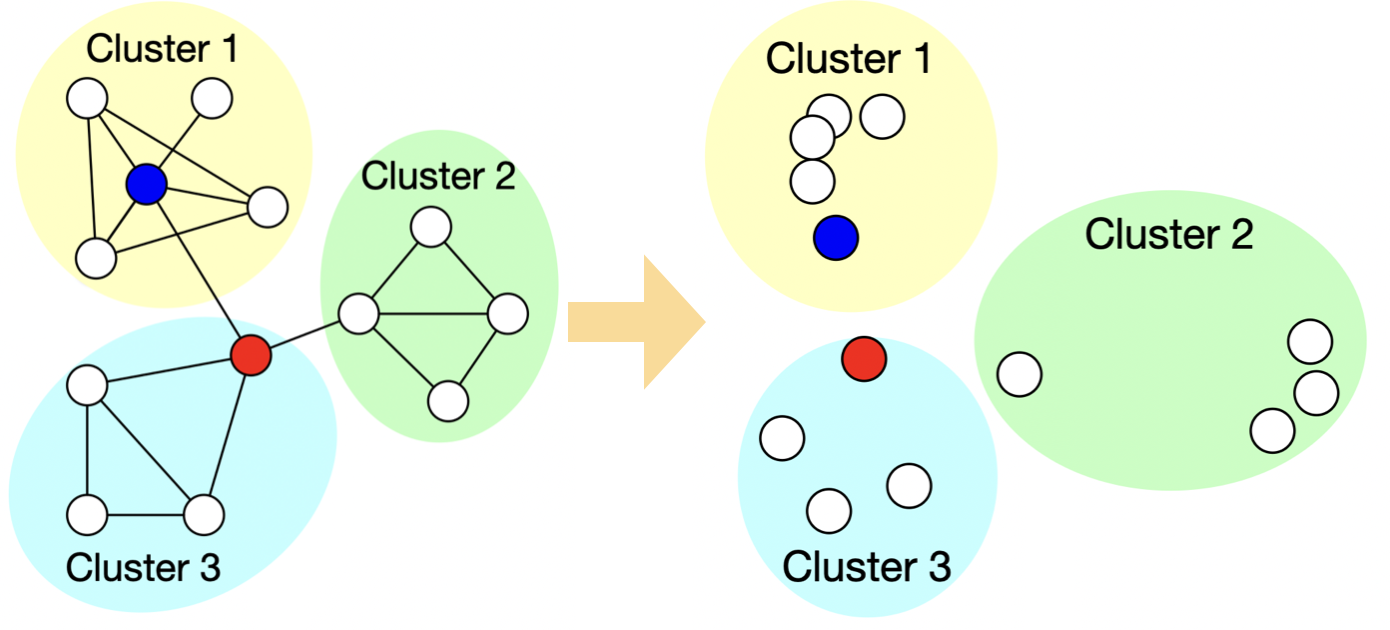}
    \caption{Left: Input graph. Right: Its node embeddings}
 \end{subfigure}
 \begin{subfigure}[t]{0.49\textwidth}
  \includegraphics[width=\textwidth]{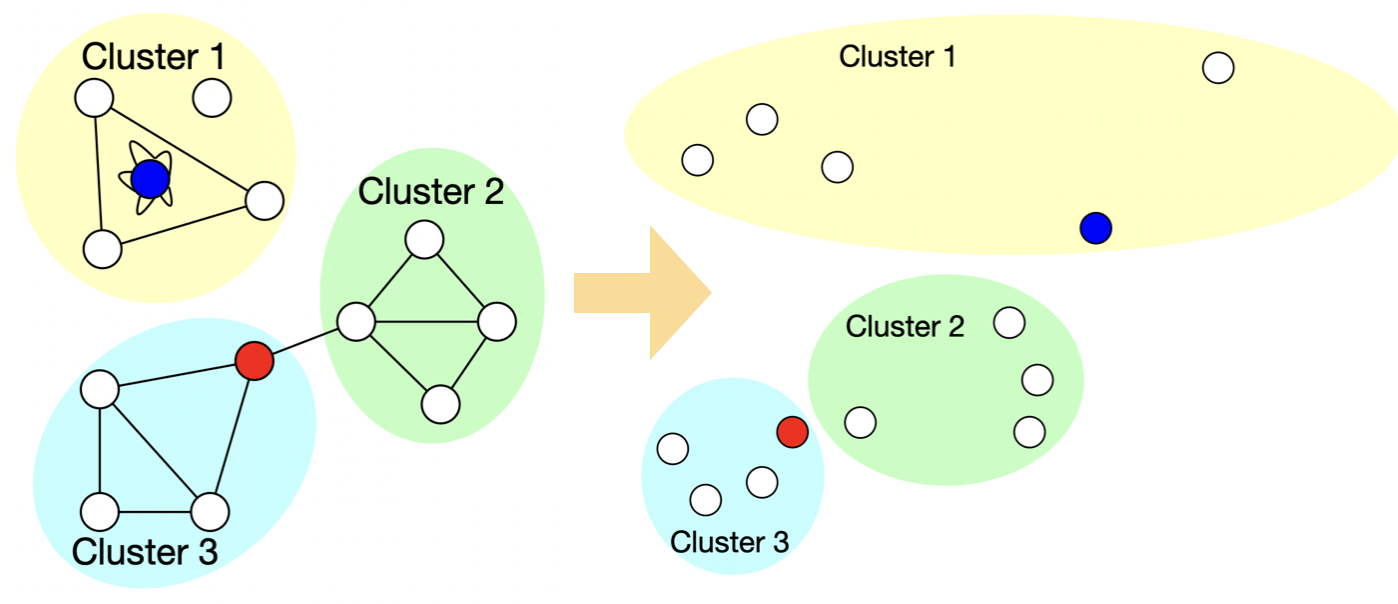}
    \caption{Left: Perturbed graph by rewiring edges of the highest degree node (blue). Right: Its embeddings}
 \end{subfigure}  
 \begin{subfigure}[t]{0.49\textwidth}
  \includegraphics[width=\textwidth]{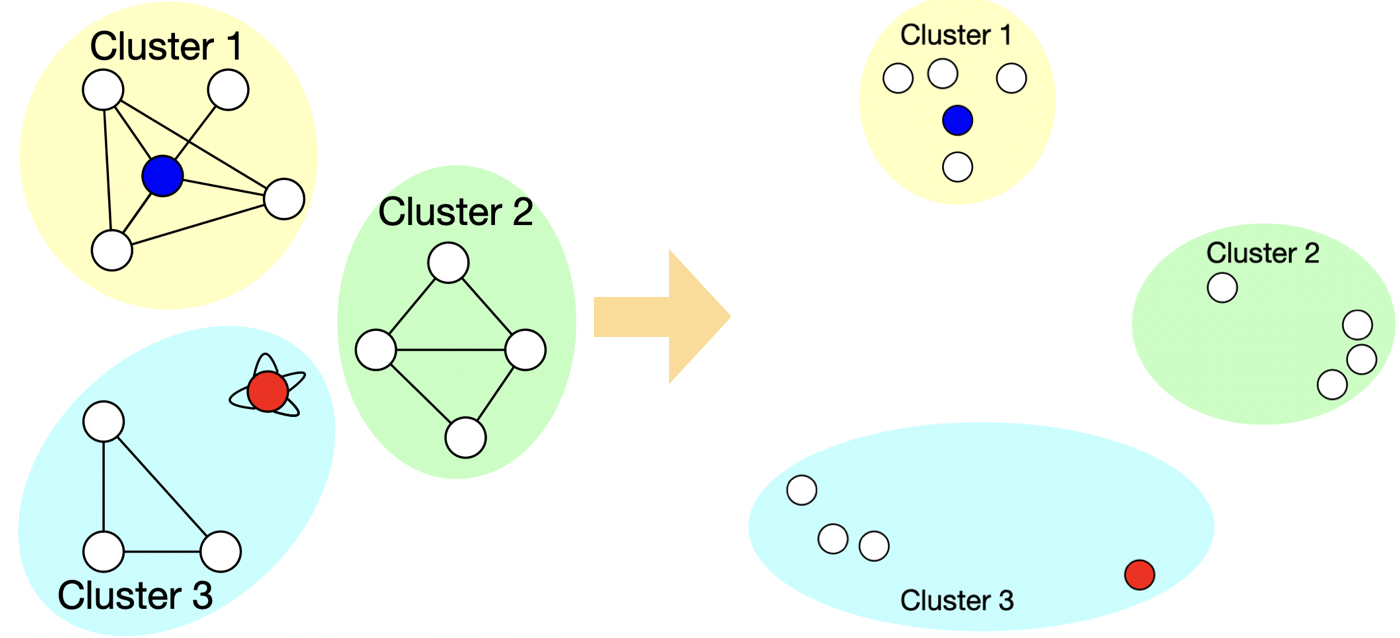}  
    \caption{Perturbed graph by rewiring edges of highest-bridgeness node (red) and its embeddings}
 \end{subfigure}  
  \caption{Illustrative example of the global effect of the highest-bridgeness node. (Best viewed in color.)}
  \label{fig:theorem:example}
\end{figure*}

\subsection{Illustrative Example} In Theorem \ref{theorem:n2v:bridge}, node-level explanations are theoretically the same as the highest-bridgeness nodes. Figure \ref{fig:theorem:example} shows an illustrative example, where the highest-bridgeness node, which is colored red, has the greatest effect on the learned embeddings. Figure \ref{fig:theorem:example}(a) shows the input graph and its node embeddings. The cluster assignments are from the spectral clustering on the embedding vectors. When we perturb the edges of the blue node (thus, the highest-degree node), the representations of the nodes of Cluster 1 move a lot but relative positions of Clusters 2 and 3 are still preserved as in Figure \ref{fig:theorem:example}(b). We note that cluster assignments still follow the initial assignments in Figure \ref{fig:theorem:example}(a). Meanwhile, when we perturb the edges of the red node (thus, the highest-bridgeness node) as in Figure \ref{fig:theorem:example}(c), relative positions across all clusters are changed because all clusters lose information to determine relationships with other clusters. Therefore, the sum of pairwise distances is changed more compared with the case of Figure \ref{fig:theorem:example}(b). In other words, according to Definition \ref{def:pairwisenodeimp}, the red node becomes the \textit{global} node-level explanation.

%!TEX root = main.tex

\section{Gradient-Based Method to  Identify High-Bridgeness Nodes}\label{sec:graphgd_wgd}

In the previous section, we showed the relationship between the bridgeness of a node and a node-level explanation in embeddings such as DeepWalk. However, existing methods to find bridge nodes are computationally intensive, $O(|E|+|V|^{2.367})$, owing to inter-community degree computation ($O(|E|)$) performed after spectral graph clustering ($O(|V|^{2.367})$). In particular, spectral clustering relies on eigenvector decomposition (EVD); the time complexity of EVD is $O(|V|^{2.367})$ or $O(nnz(A))$ \citep{garber2016faster}. Hence, $O(nnz(A))=O(A^2)$ if $A$ is dense, where $nnz(A)$ denotes the number of non-zeros in $A$.

To mitigate this burden, we propose a gradient-based method to find the explanation; this further enables us to find its top-$q$ node-level explanations. First, we derive a measure by exploiting the magnitude of the gradient, which is also called \textit{sensitivity} \citep{simonyan2013deep}. 

We note that the magnitude is measured after learning the corresponding embedding model\footnote{While learning node embeddings in a Skip-gram-based method such as DeepWalk or LINE, the magnitude of the gradient updates does not converge to a very small value. This is because although relative node positions converge, absolute node locations in the embedding space never converge in practice because of the pairwise update nature of hierarchical softmax or negative sampling \citep{word2vec,line}. 
}
To measure  \textit{sensitivity} for embedding models, we outline an initial gradient-based measure, which we call \textbf{GRAPH-GD}. Again, $W \in \mathcal{R}^{|V| \times k}$ is a learned latent node representation from model $\mathcal{M}$ and $\vec{w}_i \in \mathcal{R}^{k}$ is the embedding vector of $v_i \in V$.

\vspace{2mm}

\begin{equation*}
\resizebox{0.55\hsize}{!}{%
    $\text{\textbf{GRAPH-GD}}(v_b, W, \mathcal{N}_{G}, \mathcal{L}_{\mathcal{M}} ) =  \text{MEAN}\left\{\left| \frac { \partial \mathcal{L}_{\mathcal{M}}(v_b, v_i) }{ \partial \vec{w}_i}\right|, v_i \in \mathcal{N}_{G}(v_b|\psi)\right\}$
}
\label{equ:graphgd}
\end{equation*}
\vspace{1mm}

\noindent Here, $\mathcal{N}_{G}(v_b|\psi)$ returns the neighbors in $G$ of a given node $v_b$, with $\psi$ denoting the number of sampled neighbors from the full neighbor set. We use $\mathcal{L}_{\mathcal{M}}$ to denote the loss function of the  embedding model $\mathcal{M}$, which takes a target (input) node and context node, $v_b$ and $v_i$, respectively, and $ \partial \mathcal{L}_{\mathcal{M}} / \partial \vec{w_i}$ is a \textit{sensitivity} function to return the magnitude of its gradient w.r.t. any $\vec{w}_i \in W$. For example, the loss function of hierarchical softmax-based embedding methods (e.g., DeepWalk and Struc2Vec) has $\mathcal{L}_{hs}(v_b, v_i) = -\log  f(v=v_i|v_b)$, where $f$ is a learned function from the embedding model to calculate the likelihood of $v_i$ given $v_b$ to appear. 
In negative sampling-based embedding methods such as LINE and Node2Vec, loss functions can be generalized as $\mathcal{L}_{ns}(v_b, v_i) = -\log  \sigma(\vec{w}_i^T \vec{w}_b) - \sum_{v_z \in V_{neg}} \log  \sigma(-\vec{w}_z^T \vec{w}_b)$, where $V_{neg}$ is a sampled negative node set. This implies that our GRAPH-GD can be applied to most hierarchical softmax and negative sampling-based embedding methods. The magnitudes are aggregated over neighbors of $v_b$ and, in this work, we uniformly sample a fixed-size ($\psi$) set of neighbors for $\mathcal{N}_{G}(v_b|\psi)$, instead of using full neighborhood sets $\mathcal{N}_G(v_b)$.  Section \ref{sec:rel:graphgd:bridge} describes the theoretical relationship between GRAPH-GD and bridge nodes.

\textcolor{black}{
GRAPH-GD tries to identify bridge nodes by computing the average of the gradient updates' magnitude between a potential bridge node $v_b$ and one of its neighboring (context) nodes $v_i \in \mathcal{N}_{G}(v_b|\psi)$. However, gradient magnitude-based explainability measures are prone to capture noise rather than relevant signals~\cite{kindermans2019reliability,tomsett2020sanity,pope2019explainability}, resulting in high variance. To overcome this issue, instead of equally weighting the gradients, we take into account the degree to which each neighboring node $v_i$ is related to both the potential bridge $v_b$ and the cluster that $v_b$ is bridging. This allows us to assign higher weights to more informative neighboring nodes and better identify high bridgeness nodes. To address this, we propose a weighted aggregation method, \textbf{GRAPH-wGD}.}

%\textcolor{black}{
%GRAPH-GD attempts to discover bridge nodes by averaging the magnitudes of gradient updates between a potential bridge node $v_b$ and one of its neighboring (context) nodes $v_i \in \mathcal{N}_{G}(v_b|\psi)$. 
%%However, gradient magnitude-based explainability measures are often affected by spurious high-degree nodes, resulting in noisy explanations with large variance~\cite{pope2019explainability}.
%However, gradient magnitude-based explainability measures often capture noise rather than obtain relevent signal~\cite{kindermans2019reliability,tomsett2020sanity,pope2019explainability}, which lead to high variance. To address this issue, instead of summing the gradients with equal weights, we take into account the degree to how each neighboring node $v_i$ is associated to both the potential bridge $v_b$ and the cluster, which is being bridged by $v_b$.
%% and a new weighted aggregation approach to adjust each gradient magnitude. 
% This allows us to put higher weights to more informative neighboring nodes for identifying high bridgeness nodes. To support this, we propose a weighted aggregation method, \textbf{GRAPH-wGD}.}
% 
\vspace{2mm}

\begin{equation*}
\resizebox{0.67\hsize}{!}{%
    $\text{\textbf{GRAPH-wGD}}(v_b, W, \mathcal{N}_{G}, \mathcal{L}_{\mathcal{M}} ) \!=  \text{MEAN}\left\{h(v_b, v_i) \!\cdot\! \left|\frac { \partial \mathcal{L}_{\mathcal{M}}(v_b, v_i) }{ \partial \vec{w}_i}\right|, v_i \in \mathcal{N}_{G}(v_b|\psi)\right\}$,
}
\label{equ:graphwgd}
\end{equation*}

\begin{equation*}
\resizebox{0.37\hsize}{!}{%
    $h(v_b, v_i) = 1 \!+\! \overline{cos} (\vec{w}_b -  \vec{w}_{i}, \break  -\frac { \partial \mathcal{L}_{\mathcal{M}}(v_b, v_i) }{ \partial \vec{w}_i} )$,
}
\end{equation*}

\vspace{4mm}

\noindent where $\overline{cos}$ is a weight function. $\overline{cos}$ returns the cosine similarity value if the value is positive, otherwise 0. Its theoretical analysis is described in Sec. \ref{sec:rel:graphwgd:bridge} and explains how GRAPH-wGD improves on GRAPH-GD. \textcolor{black}{In our analysis, we find that when ${cos} (\vec{w}b - \vec{w}{i}, -\frac { \partial \mathcal{L}_{\mathcal{M}}(v_b, v_i) }{ \partial \vec{w}_i} )>0$, node $v_i$ is serving as a support node and demonstrating that node $v_b$ has a bridging role. Furthermore, the analysis shows that the cosine similarity-based weight function $h(v_b, v_i)$ provides more discriminatory power compared to the weight aggregation method of GRAPH-GD.}

Algorithm \ref{alg:imp:procedure} shows our procedure for discovering the top $q$ node-level explanations with \textbf{GRAPH-wGD} using the negative sampling of Skip-gram architecture. Given the learned embedding model and other inputs, GRAPH-wGD is used for computing node-level importance, and the top $q$ nodes are returned \textcolor{black}{as a global node-level explanation.}  We can get the gradient w.r.t. $\vec{w}_i$ as $ (\sigma(\vec{w}_i^T \vec{w}_b) - 1) \vec{w}_b$ from $\mathcal{L}_{ns}(v_b, v_i)$ \citep{rong2014word2vec}.

 \begin{algorithm}[tb]
 \SetAlgoLined
  \KwInput{A set of vertices $V$, Neighbor function $\mathcal{N}_{G}$, Embeddings $W$, number of nodes to return $q$, number of negative samples $\eta$} 
  
	Init an array, \textit{I}, which has $|V|$ element \;

	\For{$\forall v_b \in V$}{
	    $V_i$ = Sample $\psi$ neighbors of $v_b$ by using $\mathcal{N}_{G}(v_b|\psi)$ \;
		Set \textit{I}[b]= MEAN$\left\{h(v_b, v_i) \!\cdot\! \left| (\sigma(\vec{w}_i^T \vec{w}_b) - 1) \vec{w}_b  \right|, \forall v_i \in V_i \right\}$ \;
    }
	Return indices who have top $q$ highest scores in \textit{I} \;
 \caption{Procedure to find Top $q$ explanations using GRAPH-wGD with negative sampling}
 \label{alg:imp:procedure}
\end{algorithm}

\subsection{Complexity Analysis}

The time complexity of GRAPH-GD and GRAPH-wGD depends on the number of vertices $|V|$, the computation of gradient, and the size of neighbor sampling set $\psi$.  The computation of the gradient (as in \citep{rong2014word2vec} and \citep{word2vec}) in the Skip-gram architecture with negative sampling mainly depends on the size of the  dimension $k$ of the embedding vector, so the complexity is $O(k \psi|V|)$. Because $k$ and $\psi$ are constants, the complexity can be reduced to $O(|V|)$. Note that the performance of our model converges when $\psi > 100$, and we set $\psi=100$ in experiments.

\vspace{-2mm} 

\section{Theoretical Analysis}\label{sec:ta}
\subsection{Relationship between GRAPH-GD and Bridge Nodes}
\label{sec:rel:graphgd:bridge}
Let $v_j$ be the context node of $v_i$. We assume that the context node is from $\mathcal{N}_{G}({v_i}|\psi)$. Let $\vec{w}_i$ be the embedding of $v_i$.
We assume that $\left|\frac { \partial \mathcal{L}(v_j, v_i) }{ \partial \vec{w}_i}\right| = dist(\vec{w}_i, \vec{w}_j )$, where the  function $dist$ measures the distance between $\vec{w}_i$ and $\vec{w}_j$. The distance is used to decide the error to update the embedding vector for $\vec{w}_i$. To show how GRAPH-GD finds bridge nodes, we need to assume that the learned embedding is {\em good embedding} as follows.

\begin{definition}[Good embedding] Let $W$ be an embedding learned from a graph $G$ with model $\mathcal{M}$ and let $C$ be a set of (arbitrary) clusters from $G$. Let  $\vec{w}_i$ be the embedding of node $v_i$. 
We refer to $W$ as a {\em good embedding} w.r.t.~$C$ when it satisfies the following condition: for all $v_i, v_j, v_k \in V$, if $v_i, v_j \in C_b$ and $v_k \in C_{n (\neq b)}$, then  ${dist}(\vec{w}_i, \vec{w}_j)  < dist(\vec{w}_k,\vec{w}_j)$.

\label{def:goodembedding}
\end{definition}

This definition does not imply a perfect clustering according to the gold standard. After clustering on the embeddings, provided that clusters are separated sufficiently in the embedding space, the condition is satisfied. The good embedding assumption is applicable to both hierarchical softmax and negative sampling-based embeddings for learning Skip-gram architectures. For those embeddings, nodes in the same cluster are likely to be placed close together; therefore, their gradient update is smaller than the gradient update of the opposite case. Based on the abovementioned definition, when comparing two nodes with the same degree, GRAPH-GD identifies the node with higher  \textit{bridgeness}. The notation ($W$, $\mathcal{N}_{G}$, $\mathcal{L}_{\mathcal{M}}$) in GRAPH-GD does not change in this analysis, so we neglect it for clarity.

\begin{lemma}
Let $G$ be a graph with a set of (arbitrary) clusters $C$ and associated embeddings $W$. Let $v_b \in C_b$ be a bridge node and let $v_c \in C_b$ be a node with the same degree as $v_b$, but fewer inter-cluster edges (thus, lower \textit{bridgeness}). Here, $\text{B-Cluster}(v_b)$ denotes the set of clusters that are bridged by $v_b$. We assume that $\text{B-Cluster}(v_b) \supseteq \text{B-Cluster}(v_c)$ and $W$ is a {\em good embedding} as in Definition \ref{def:goodembedding}. Then, GRAPH-GD($v_b$) $>$ GRAPH-GD($v_c$).
\label{lemma:graphgd:bridge}
\end{lemma}

% \vspace{-2mm}
\subsection{Comparison between GRAPH-wGD and GRAPH-GD}
\label{sec:rel:graphwgd:bridge}

In this section, we show how GRAPH-wGD helps to identify nodes that have higher \textit{bridgeness} than GRAPH-GD. First, we state the definition of a \textit{support node}. 

\vspace{-1mm}
\begin{definition}[Support node] Assume that in graph $G$ with (arbitrary) clusters $C$, $v_b$ is a bridge node. Let $\text{B-Cluster}(v_b)$ be the set of clusters that are bridged by $v_b$. Then let cluster $C_i \subset \text{B-Cluster}(v_b)$, let $\vec{w}_b$ be the embedding of $v_b$, and let $\vec{w}_{C_i}$ be the centroid of nodes in $C_i$ in the embedding space.
Let $cos$ be a cosine distance function. If $cos(\vec{w}_b - \vec{w}_i, \vec{w}_b-\vec{w}_{C_i}) > 0$, then $v_i$ is a support node for giving $v_b$ the bridging role for $C_i$. 
\label{def:support}
\end{definition}
\vspace{-1mm}

We show that support nodes provide more meaningful evidence for identifying the bridge role of $\vec{w}_b$ for $C_i$ than non-support nodes, which can potentially enhance GRAPH-GD. Definition \ref{def:dweight} and Lemma \ref{lem:support:dweight} in Section~\ref{sec:proofs} show that the directional weight can find support nodes. The following theorem shows that GRAPH-wGD helps to identify nodes that have higher \textit{bridgeness} than GRAPH-GD because the difference in Graph-wGD is larger than that in Graph-GD when two nodes have different \textit{bridgeness}. We assume that all neighbors are returned from $\mathcal{N}_G$ without sampling (i.e., $\psi=max$(node degree)). 

\vspace{-1mm}
\begin{theorem}
Let $G$ be a graph with a set of (arbitrary) clusters $C$ and associated embeddings $W$. Let $v_b \in C_b$ be a bridge node with at least one support node in $\mathcal{N}_{G}(v_b)$, let $v_c \in C_b$ be a node with the same degree but lower \textit{bridgeness} than $v_b$, and let $\mathcal{N}_{G}(v_c)$ have no support node. Assume that $\text{B-Cluster}(v_b) \supseteq \text{B-Cluster}(v_c)$ and that $W$ is a {\em good embedding} as in Definition \ref{def:goodembedding}.
If, for any $C_i \!\subset\! \text{B-Cluster}(v_b)$, the sum of gradient updates of $\{(v_j,v_k)|\: \forall v_j, v_k \!\in\! \ C_i\}$ converges to the centroid $w_{C_i}$ (due to dense edges among nodes in $C_i$), then GRAPH-wGD($v_b$) $>$ GRAPH-GD($v_b$) $>$ GRAPH-GD($v_c$) $=$ GRAPH-wGD($v_c$). 

\label{theorem:graphwgd:bridge}
\end{theorem}

\vspace{-1mm}
\noindent From this theorem, we can expect GRAPH-wGD to assign node-importance scores that correlate more closely with \textit{bridgeness}.
\vspace{-4mm}

%!TEX root = main.tex

\section{Experiments}\label{sec:exp}

To evaluate our method, five real-world networks (Karate Network \citep{zachary1977information}, Wikipedia (Wiki) \citep{mahoney2011large}, BlogCatalog \citep{zafarani2009social}, Enron network data \citep{diesner2005exploration}, and NeurIPS data\footnote{Available at: https://www.kaggle.com/datasets/benhamner/nips-papers}) were used. First, Zachary's Karate Network is a social network of a university karate club. The network was used to understand what our GRAPH-wGD finds and to observe how the found nodes were correlated to the rank by true \textit{bridgeness} scores. Second, Wikipedia (Wiki) \citep{mahoney2011large} is a co-occurrence network of words appearing in a Wikipedia dump, which is composed of 4,777 nodes, 184,812 edges, and 40 labels. Third, BlogCatalog \citep{zafarani2009social} is a network of social relationships of the bloggers listed on the BlogCatalog website. It had 10,312 nodes, 333,983 edges, and 39 labels. Fourth, the Enron network \citep{diesner2005exploration} is composed of 150 nodes and 1,853 edges (after unifying 517,431 emails). For class labels, employees in management positions were used to evaluate our method. Fourth, the NeurIPS dataset was pre-processed from a Kaggle repository (benhamner/nips-papers) and co-authorship edges were extracted from papers between 1987 and 2017. It had 9,784 vertices and 22,198 edges.

\subsection{Evaluation Measures}

In the previous sections, we found that the node-level explanations (i.e., most-important nodes) were theoretically related to the highest-bridgeness nodes that connected across different clusters in the embedding space. To capture nodes that have high \textit{bridgeness} scores efficiently, a gradient-based explanation model GRAPH-wGD was proposed. We evaluated GRAPH-wGD using \textit{rank correlation}, \textit{node importance} (NI), and \textit{prediction change} (PC).

\vspace{1mm}
\paragraph{Spearman's Rank Correlation}
The Spearman correlation coefficient is defined as the Pearson correlation coefficient between the rank variables. This was used to evaluate how similarly our method ranks compared to the ranking by \textit{bridgeness}.

\vspace{1mm}
\paragraph{\textit{Node Importance}}
NI was used to determine the effect of the found global explanations (i.e., top-$q$ important nodes) in the embedding spaces.
To evaluate the actual effect of the important nodes returned by each algorithm, we used the node {\em Imp} function from Definition \ref{def:newnodeimp}. However, instead of calculating the value for each node $v_b$, we selected the top $z\%$ of returned ``important’’ nodes, which we refered to as $\mathbf{V}_{imp}$, and calculated $Imp(\mathbf{V}_{imp} | W, W'_{-\mathbf{V}_{imp}})$ with $m=0.05 \cdot |V|$, using the perturb function (Definition \ref{def:perturb}) with $\alpha=0.5$.

\vspace{1mm}
\paragraph{Prediction Change}
PC was used to determine the effect of the found global explanations in downstream tasks. In other words, this measured how much the predicted probability of a supervised model learned from embeddings changed after perturbing important nodes and re-learning the embeddings. The supervised prediction model was trained by using the embeddings of nodes in $V_{train}$ and their labels. 
Let $p(\hat{y}_i=c | W)$ be a predicted probability of $v_i$ for a class label $c$ from the model given the node embedding $W$ of $G$. Let $p(\hat{y_i}=c | W'_{-\mathbf{V}_{imp}})$ be the new prediction for $v_i$ based on the embedding $W'_{-\mathbf{V}_{imp}}$ of $G'_{-V_{imp}}$, using the perturb function (Definition \ref{def:perturb}) with $\alpha=0.5$. The PC score was computed as follows: $PC(\mathbf{V}_{imp}|W, W'_{-\mathbf{V}_{imp}}) = \frac{1}{|V_{test}|} \cdot \sum_{v_i \in V_{test}} \sum_{c \in C} {||(p(\hat{y}_i=c | W)  - p(\hat{y}_i=c | W'_{-\mathbf{V}_{imp}})||})$.

\subsection{Comparison of Models used to Determine Node Importance}

\paragraph{Bridgeness} We first learned node embeddings and then captured scores of the nodes \textit{bridgeness}, which were based on found clusters on the learned embedding using spectral clustering. Here \textit{bridgeness} is defined in Equation (\ref{equ:bridgeness}).

\vspace{1mm}
\paragraph{Degree} Node degrees were used for node importance scores, and we decided node-level explanations from the highest-degree nodes.

\vspace{1mm}
\paragraph{Personalized PageRank (PPR) } PPR was also used as a baseline method. We used default parameters in NetworkX 2.1.

\vspace{1mm}
\paragraph{Bridge-Indicator} A graph-based bridgeness scoring algorithm \citep{jensen2015detecting} was used.

\vspace{1mm}
\paragraph{Greedy} The sum of absolute changes in predictions over all labels from an input graph $G$ and a perturbed graph $G'_{-v_i}$ by $v_i$ were leveraged for the importance score of $v_i$ as in Equation (\ref{equ:greadyalg}). 
The perturbation strategy still followed Definition \ref{def:perturb}.

\vspace{1mm}
\paragraph{LIME} LIME \citep{lime} can measure how locally important nearby nodes are to predict the label of an instance. However, LIME is not designed for generating global node-level explanations for unsupervised node embeddings. For a fair comparison, similar to our GRAPH-wGD, we computed MEAN$\{\text{LIME}(v_c|v_i), v_c \in \mathcal{N}(v_i)\}$ by regarding $v_c$ as local features for $v_i$. $\mathcal{N}_{G}(v_i|\psi)$ also returns the neighbors with the same size of neighbor set, $\psi$.

\vspace{1mm}
\paragraph{GNNExplainer} Although GNNExplainer \citep{gnnexplainer} was also not designed for unsupervised node embeddings, we could still use the same objective function to quantify the importance of each node by leveraging the available labels. Thus, the mutual information between predictions from an input graph and a perturbed graph by $v_i$ was leveraged for the importance score of $v_i$. In this case, we assumed that nodes that return low mutual information values were more important.

\vspace{1mm}
\paragraph{GENI} GENI \citep{geni} is a supervised importance score regression model and used for predicting bridgeness scores. 

For input scoring networks, known (true) bridgeness scores were fed for nodes in the training node-set.

\subsection{Hyperparameters and Cross-Validation Settings}

Two embedding methods (DeepWalk and LINE) were used for evaluating our GRAPH-wGD. For learning DeepWalk for Wiki and BlogCatalog, we set all the parameters as in \citep{node2vec} with hierarchical softmax because they experimented with the same datasets. For the Enron and Karate networks, we set the random walk size to 5, window size to 5, and size of the embedding vector to 8. For NeurIPS, all parameters were set as in the Wiki dataset. For learning LINE, we also used the same sizes of embedding vectors. The learning rates for both embeddings were set as in the original papers for all datasets. To obtain the magnitude of the gradient on LINE, the gradient of first-order proximity embedding was exploited. For the perturbed ratio, we set $\alpha=0.5$. When we perturbed the edges, clusters were found from spectral clustering on the learned embedding and the number of clusters was obtained from the number of their class labels. For classification, a fully connected neural network with two layers was used for predictions. The hidden node size was defined as two times the class label size of each dataset. 
	
Regarding node-set splits, the top 40\% of the high-degree nodes were chosen for a testing node set $V_{test}$ (i.e., for rank correlation, NI, PC). Then 10\% of the nodes from the remaining node set were randomly selected for a training node set $V_{train}$ and validation node set $V_{valid}$. We note that the training and validation node sets were used only for Greedy, LIME, GNNExplainer, and GENI, and the results were averaged after three different random training/validation node-set selections. Importance scores from other bases such as Degree, PPR, and Bridgeness were estimated from the graph input. Therefore, after we computed the scores, the results on testing node sets were reported.

%!TEX root = main.tex

\begin{figure}[!t]
     \centering
    %  \vspace{-5mm}
     \includegraphics[width=0.42\textwidth,trim=50 56 50 50, clip]{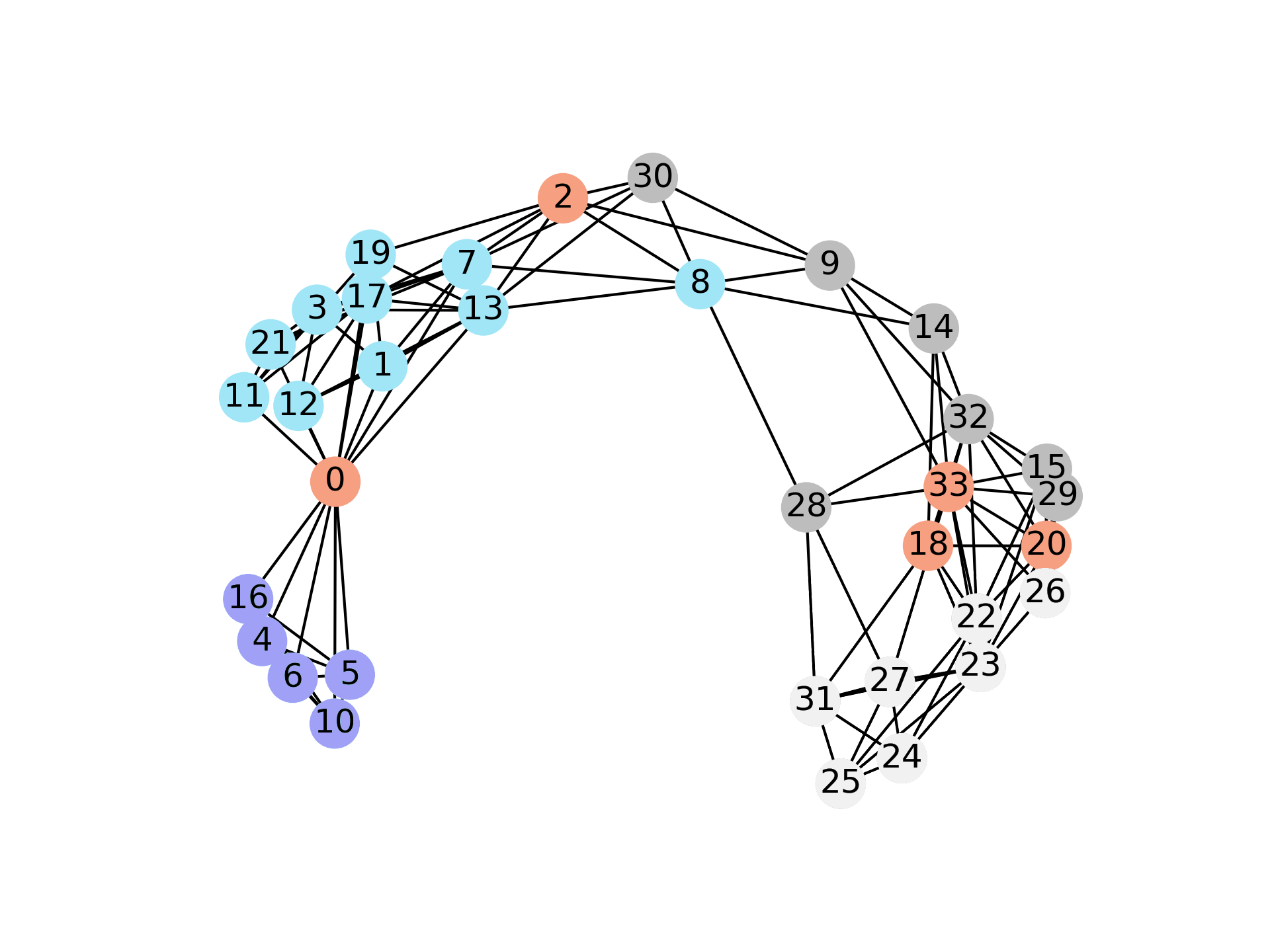}
    %  \vspace{-2mm}
     \caption{The top-five most-important nodes in DeepWalk embedding by GRAPH-wGD are colored orange. The visualization of clusters is from spectral clustering with t-SNE \cite{tsne}.}     
	 \label{fig:syn:karate:vis}
 %\vspace{-2mm}
\end{figure}

\begin{table}[!t]
    	\caption{Spearman's rank correlation in the Karate Club network. (Bold indicates the best score.)} 
	\label{tab:karate:ranking}

    \begin{minipage}[h]{.95\textwidth }
		\centering
        \scalebox{0.9}{             		
		\begin{tabular}{c|cc}
		       & \multicolumn{2}{c}{Karate}  \\
		       & DW       & LINE              \\ \hline
		        Greedy & 0.175 & 0.350 \\
		        LIME & 0.175 & 0.359  \\
		        GNNExplainer & 0.175 & 0.350  \\ 
		        GENI & 0.633 & 0.350  \\ \hline
		        GRAPH-GD & 0.634 & 0.607 \\
		        \textbf{GRAPH-wGD} & \textbf{0.683} & \textbf{0.634} 
		\end{tabular}
        }
        \hrule height 0pt
    \end{minipage}%
\end{table}

\subsection{Experimental Results}
\paragraph{Experiment 1: Rank Correlation Analysis with Karate Graph}

In this experiment, Zachary’s Karate Club network is used to show how the ranking from our methods, GRAPH-GD and GRAPH-wGD, works by measuring rank correlation to the ranking from \textit{Bridgeness} and visualizing top nodes for qualitative analysis. Table \ref{tab:karate:ranking} shows the result of Spearman’s rank correlation coefficient ($\rho$) using two node embeddings, DeepWalk (DW) and LINE. In both embeddings, our GRAPH-wGD shows the best performance in finding bridge nodes, and $p$-values for our methods are less than 0.05 using Student’s $t$-test. 
In Figure \ref{fig:syn:karate:vis}, we can see the visualization of the Karate network with the top-five most important nodes that are found by GRAPH-wGD. Nodes 0, 2, 33, 18, and 20 are chosen as explanations by GRAPH-wGD, and they are serving as bridge nodes on the learned latent space. The result also corresponds to bridge nodes found in other studies (e.g., \citep{wang2011identifying}).
\vspace{1mm}

\begin{table}[!t]
    	\caption{Spearman's rank correlation in Wiki and BlogCatalog. (Bold indicates the best score.)} 
	\label{tab:wikiblog:ranking}

    \begin{minipage}[h]{.95\textwidth }
		\centering
        \scalebox{0.9}{             		
		\begin{tabular}{c|cc|cc}
		       & \multicolumn{2}{c|}{Wiki}       & \multicolumn{2}{c}{BlogCatalog} \\
		       & DeepWalk       & LINE & DeepWalk       & LINE           \\ \hline
		        Greedy & 0.160 & 0.056 & 0.140 & 0.010 \\
		        LIME  & 0.177 & 0.072 & 0.185 & 0.141 \\
		        GNNExplainer  & 0.193 & 0.098 & 0.196 & 0.160 \\ 
		        GENI  & 0.391 & 0.073 & 0.354 & 0.072 \\ \hline
		        GRAPH-GD  & 0.197 & 0.158 & 0.268 & 0.141 \\
		        \textbf{GRAPH-wGD} & \textbf{0.392} & \textbf{0.458} & \textbf{0.368} & \textbf{0.334}
		\end{tabular}
        }
        \hrule height 0pt
    \end{minipage}%
\end{table}

\paragraph{Experiment 2: Correlation analysis and measurement of the effects of found node-level explanations in Wiki and BlogCatalog}

Table \ref{tab:wikiblog:ranking} lists Spearman’s rank correlation to bridge nodes on Wiki and BlogCatalog. In the experiments, our GRAPH-wGD outperforms other methods. In particular, our GRAPH-wGD shows better performance than GRAPH-GD. \textcolor{black}{GRAPH-GD's aggregation method, which sums up the gradient magnitudes from neighboring nodes with equal weights, results in a lower rank correlation compared to GRAPH-wGD.} This indicates that capturing support nodes helps to find bridge nodes. Other prediction probability-based methods, such as Greedy and GNNExplainer, are not highly correlated to \textit{Bridgeness}-based ranking. GNNExplainer and GENI show better correlations than others, but the correlations are still lower than the results from GRAPH-wGD.

In Tables \ref{tab:wiki:rolechange} and \ref{tab:blog:rolechange},  NI is reported to show the number of new neighbors appearing in the neighborhood of a node after perturbing edges of 3\%, 5\%, and 7\% top nodes of each method and re-learning the embedding. Here we note that \textit{Bridgeness} provides theoretical upper bounds to all methods as expected from Theorem \ref{theorem:n2v:bridge}. However, it is computationally intensive to calculate \textit{Bridgeness} and the top nodes identified more efficiently by GRAPH-wGD perform similarly. Our approach GRAPH-wGD produces better results. 

Tables \ref{tab:wiki:predchange} and \ref{tab:blogi:predchange} list the changes in prediction after the graph is perturbed and embedding is relearned. In the tables, 3\%, 5\%, and 7\% of important nodes, which are found from each method, are perturbed to observe the effect w.r.t.\ the change in prediction. Compared with LINE, DeepWalk shows relatively smaller changes because the two datasets are relatively dense and the characteristics of random walks diminish the effect of perturbation.  Again, GRAPH-wGD  shows more changes in neighbors than other alternatives. LIME and GNNExplainer were designed to provide local explanations per specific target nodes, so their local decisions did not find nodes that change the predictions for all other nodes. \textcolor{black}{GRAPH-GD and GRAPH-wGD can effectively identify high bridgeness nodes, as demonstrated in Lemma \ref{lemma:graphgd:bridge}, leading to high Node Importance (NI) and Prediction Change (PC) scores. While GRAPH-wGD enhances the discriminatory power through weight adjustment, both GRAPH-GD and GRAPH-wGD exhibit similar performance in DeepWalk, especially in the BlogCatalog dataset, which has an average degree 6 times higher than Wiki. This allows for many candidate nodes to be fed into the random walks, resulting in good gradient-based signals for learning the node embeddings. However, in the LINE algorithm, which only considers closer neighbors, GRAPH-GD may struggle to effectively handle the variance from noisy neighboring nodes, leading to a lower performance compared to GRAPH-wGD.}

\begin{table*}[!t]
    \begin{minipage}[h]{.95\textwidth }
        \centering
                \caption{NI after the perturbation of the top 3\%, 5\%, and 7\% of nodes in Wiki. \\(Bold indicates the best score and \textit{Bridgeness} provides theoretical upper bounds.)} 
        \label{tab:wiki:rolechange}

         \captionsetup{justification=centering}
        \scalebox{0.9}{     
            \begin{tabular}{c|ccc|ccc}
            & \multicolumn{3}{c|}{DeepWalk}                    & \multicolumn{3}{c}{LINE}                        \\
            & 3\%           & 5\%           & 7\%          & 3\%           & 5\%           & 7\%          \\ \hline
Bridgeness & 0.706 & 0.705 & 0.707 & 0.644 & 0.659 & 0.665 \\  \hline \hline
Degree & 0.260 & 0.295 & 0.290 & 0.375 & 0.399 & 0.411 \\
PPR & 0.284 & 0.284 & 0.216 & 0.381 & 0.405 & 0.421 \\ 
Bridge-Indicator & 0.615 & 0.575 & 0.577 & 0.400 & 0.407 & 0.426 \\ \hline 
Greedy & 0.544 & 0.572 & 0.513 & 0.549 & 0.574 & 0.591 \\ 
LIME & 0.514 & 0.527 & 0.584 & 0.539 & 0.582 & 0.611 \\
GNNExplainer & 0.680 & 0.679 & 0.674 & 0.548 & 0.563 & 0.572 \\ 
GENI & 0.609 & 0.623 & 0.609 & 0.435 & 0.450 & 0.456 \\ \hline
GRAPH-GD & 0.682 & \textbf{0.689} & 0.691 & 0.523 & 0.534 & 0.555 \\
\textbf{GRAPH-wGD} & \textbf{0.689} & \textbf{0.689} & \textbf{0.697} & \textbf{0.564} & \textbf{0.590} & \textbf{0.630} \\

            \end{tabular}

        }            
        \hrule height 0pt
    \end{minipage}%
\end{table*}    
\begin{table*}[!t]
    \begin{minipage}[h]{.95\textwidth }
        \centering
                 \caption{NI after the perturbation of the top 3\%, 5\%, and 7\% of nodes in BlogCatalog. \\(Bold indicates the best score and \textit{Bridgeness} provides theoretical upper bounds.)} 
        \label{tab:blog:rolechange}
        
        \captionsetup{justification=centering}
       
        \scalebox{0.9}{     
            \begin{tabular}{c|ccc|ccc}
                   & \multicolumn{3}{c|}{DeepWalk}                    & \multicolumn{3}{c}{LINE}                         \\
                   & 3\%           & 5\%           & 7\%          & 3\%           & 5\%           & 7\%          \\ \hline
Bridgeness & 0.682 & 0.642 & 0.668 & 0.624 & 0.628 & 0.632 \\ \hline \hline
Degree & 0.422 & 0.475 & 0.504 & 0.346 & 0.351 & 0.359 \\
PPR & 0.328 & 0.350 & 0.383 & 0.343 & 0.349 & 0.357 \\ 
Bridge-Indicator & 0.255 & 0.300 & 0.331 & 0.263 & 0.266 & 0.271 \\ \hline 
Greedy & 0.508 & 0.509 & 0.527 & 0.498 & 0.510 & 0.518 \\
LIME & 0.463 & 0.510 & 0.521 & 0.484 & 0.510 & 0.523 \\
GNNExplainer & 0.443 & 0.521 & 0.546 & 0.497 & 0.509 & 0.517 \\ 
GENI & 0.272 & 0.359 & 0.365 & 0.308 & 0.321 & 0.319  \\ \hline
GRAPH-GD & 0.639 & \textbf{0.654} & 0.659 & 0.487 & 0.488 & 0.502 \\
\textbf{GRAPH-wGD} & \textbf{0.650} & 0.651 & \textbf{0.661} & \textbf{0.510} & \textbf{0.518} &\textbf{0.541} \\

            \end{tabular}
        }            
        \hrule height 0pt

   \end{minipage}%

\end{table*}

\begin{table*}[!t]
    \begin{minipage}[h]{.97\textwidth }
        \centering
         \caption{PC after the perturbation of \\the top 3\%, 5\%, and 7\% of nodes in Wiki. \\(Bold indicates the best score and \textit{Bridgeness} provides theoretical upper bounds.)}
        \label{tab:wiki:predchange}
        \captionsetup{justification=centering}
        \scalebox{0.9}{     
            \begin{tabular}{c|ccc|ccc}
                & \multicolumn{3}{c|}{DeepWalk}                 & \multicolumn{3}{c}{LINE}                         \\
                       & 3\%         & 5\%          & 7\%         & 3\%           & 5\%           & 7\%          \\ \hline
Bridgeness & 2.03 & 1.65 & 1.84 & 37.23 & 38.38 & 38.88 \\ \hline \hline
Degree & 1.00 & 1.46 & 1.41 & 25.28 & 25.55 & 27.44 \\
PPR & 1.12 & 1.04 & 1.17 & 26.45 & 26.86 & 28.00 \\ 
Bridge-Indicator & 1.35 & 1.36 & 1.45 & 34.60 & 34.09 & 37.41 \\ \hline 
Greedy & 1.11 & 1.06 & 1.14 & 29.30 & 29.99 & 33.61 \\
LIME & 1.15 & 1.10 & 1.19 & 25.34 & 32.41 & 37.84 \\
GNNExplainer & 1.35 & 1.37 & 1.42 & 25.94 & 26.57 & 30.63 \\ 
GENI & 1.35 & 1.36 & 1.45 & \textbf{35.26} & 35.47 & 37.44 \\ \hline
GRAPH-GD & 1.45 & 1.09 & 1.25 & 24.31 & 26.94 & 34.06 \\
\textbf{GRAPH-wGD} & \textbf{1.98} & \textbf{1.59} & \textbf{1.68} & 34.25 & \textbf{35.88} & \textbf{43.14} \\

            \end{tabular}
        }            
        \hrule height 0pt
    \end{minipage}%
\end{table*}

\begin{table*}[!t]    
    \begin{minipage}[h]{.97\textwidth}
        \centering
                \caption{PC after the perturbation of the top 3\%, 5\%, and 7\% of nodes in BlogCatalog. \\(Bold indicates the best score and \textit{Bridgeness} provides theoretical upper bounds.)} 
        \label{tab:blogi:predchange}           
	
        \captionsetup{justification=centering}
        \scalebox{0.9}{     
                        \begin{tabular}{c|ccc|ccc}
                               & \multicolumn{3}{c|}{DeepWalk}                 & \multicolumn{3}{c}{LINE}                         \\
                               & 3\%          & 5\%          & 7\%         & 3\%           & 5\%           & 7\%          \\ \hline
Bridgeness & 2.73 & 2.76 & 2.75 & 41.29 & 40.92 & 38.95 \\ \hline \hline
Degree & 2.32 & 2.15 & 2.29 & 37.54 & 36.12 & 38.09 \\
PPR & 2.30 & 2.03 & 2.27 & 38.72 & 39.10 & 37.62 \\ 
Bridge-Indicator & 2.19 & 2.10 & 2.08 & 37.21 & 37.79 & 35.58 \\ \hline 
Greedy & 2.15 & 2.27 & 2.19 & 34.87 & 38.19 & 36.52 \\
LIME & 2.20 & 1.98 & 2.33 & 37.07 & 38.14 & 38.65 \\
GNNExplainer & 2.48 & 2.54 & 2.71 & 37.50 & 37.73 & 38.38 \\ 
GENI & \textbf{2.79} & 2.27 & 2.24 & 39.24 & 39.32 & 37.66 \\ \hline
GRAPH-GD & 2.60 & 2.34 & \textbf{2.77} & 37.42 & 37.78 & 37.17 \\
\textbf{GRAPH-wGD} & 2.57 & \textbf{2.74} & 2.76 & \textbf{40.73} & \textbf{40.98} & \textbf{39.48} \\

            \end{tabular}
        }            
        \hrule height 0pt
 
    \end{minipage}%

\end{table*}

% \vspace{1mm}

\paragraph{Experiment 3: Finding explanations and measuring the effects of the found nodes in Enron}
In this experiment, we explored how the discovered explanations were associated with organizational positions. By leveraging the known position information of each person in the Enron dataset, we first measured the number of effective explanations from each method for finding people in management positions. Regarding this, we considered management people as important to form the communication graph and found the precision@$z\%$ from the ranking for each method. GRAPH-wGD showed high rank correlations ($>$0.6). In Table \ref{tab:enron:precision}, the precision@$z$(=20\% nodes) results are reported. The result shows that our GRAPH-wGD is also effective in finding people with management roles. This indicates that management roles are also highly correlated with bridging roles in communication. Table \ref{teb:top5:enron} lists the top node-level explanations from \textit{Degree} and GRAPH-wGD. Although \textit{Degree} finds the most active people such as the CEO and Chairman from the graph input view, GRAPH-wGD identifies bridging nodes (i.e., people who worked as VPs and managers) in the learned embedding space as globally important. \textcolor{black}{Our results indicate that, except for F. Sturm, GRAPH-GD shows similar performance in the top-5 ranking compared to degree-based ranking. This suggests that, due to the presence of noisy gradient information around high-degree nodes, GRAPH-GD may struggle in effectively capturing high bridgeness nodes.}

\begin{table}[!t]
    \begin{minipage}[h]{.95\textwidth }
        \centering
        	\caption{Top-five nodes in Enron}
	\label{teb:top5:enron}      

        \scalebox{0.92}{             
        \begin{tabular}{c|l}
        \multicolumn{1}{c}{}       & \multicolumn{1}{c}{Top-5 Nodes}                                                          \\ \hline
        Degree  & \begin{tabular}[c]{@{}l@{}} L. Taylor, S. Beck, J. Lavorato (CEO), \\K. Lay (Chairman), L. Kitchen (President)\end{tabular}     \\ \hline
        \textcolor{black}{GRAPH-GD}  & \begin{tabular}[c]{@{}l@{}}         \textcolor{black}{L. Taylor,  J. Lavorato (CEO), S. Beck,} \\\textcolor{black}{K. Lay (Chairman), F. Sturm (VP)}\end{tabular}     \\ \hline        
        \textbf{GRAPH-wGD} & \begin{tabular}[c]{@{}l@{}} S. Scott, F. Sturm (VP), J. Williams (VP), \\ K. Lay (Chairman), P. Allen (Manager)\end{tabular}
        \end{tabular}
        }
        \hrule height 0pt
    \end{minipage}%

\end{table}

\begin{table}[!t]
    \begin{minipage}[h]{.95\textwidth }
	\centering
	    	\caption{Precision@20\%: finding management nodes in Enron. \\(Bold indicates the best score.)} 
	\label{tab:enron:precision}

    \scalebox{0.90}{             
	\begin{tabular}{c|ll}
	                   & DeepWalk      & LINE         \\ \hline
	Greedy             & 0.25           & 0.30           \\
	LIME               & 0.40           & 0.45           \\
	GNNExplainer       & 0.45           & 0.30          \\ 
	GENI & 0.45 & 0.45 \\ \hline
	GRAPH-GD           & 0.45           & 0.45        \\
	\textbf{GRAPH-wGD} & \textbf{0.60}           & \textbf{0.75}              
	\end{tabular}
    }
    \hrule height 0pt
    \end{minipage}%
    \vspace{6mm}
\end{table}

% \vspace{20mm}

\paragraph{Experiment 4: Analysis with NeurIPS}

    \begin{table}[!t]
	    \begin{minipage}[h]{.99\textwidth }        
	        \centering
	        \caption{Top-8 Nodes in NeurIPS}
	        \label{tab:top8:neurips}        		    
	        
            \scalebox{0.92}{    
	            \begin{tabular}{c|l}
	            \multicolumn{1}{c}{}       & \multicolumn{1}{c}{Top-8 Nodes}                                                          \\ \hline
	            Degree  & \begin{tabular}[c]{@{}l@{}} G. Hinton, Y. Bengio, M. Jordan, Z. Ghahramani,\\ K. Muller,  B. Scholkop, A. Ng, R. Salakhutdinov\end{tabular}     \\ \hline
	            \textcolor{black}{GRAPH-GD}  & \begin{tabular}[c]{@{}l@{}} \textcolor{black}{Y. Bengio, G. Hinton, M. Jordan, K. Muller,}\\ \textcolor{black}{ A. McCallum, B. Scholkop, A. Ng, E. Xing, Z. Ghahramani}\end{tabular}     \\ \hline
	            \textbf{GRAPH-wGD} & \begin{tabular}[c]{@{}l@{}} R. Rosales , I. Guyon , M. Figueired, C. Scott, \\ M. Zibulevsky, A. McCallum, E. Xing, R. Vogelstein \end{tabular}
	            \end{tabular}
		        }
		        \hrule height 0pt
		    \end{minipage}%	
    \end{table}

On the NeurIPS data, we report the top-eight explanations from Table \ref{tab:top8:neurips} in the main text. For GRAPH-wGD with DeepWalk, it finds many interdisciplinary researchers such as R. Rosales (machine learning scientist who works on text mining, computer vision, graphics, and medicine) and I. Guyon (machine learning consultant who works on machine learning theory, handwriting recognition, and biomedical research for genomics, proteomics, cancer). The result is also visualized in Figure \ref{fig:neurips-graphwgd}. This visualization is from the t-SNE algorithm with the default parameters in scikit-learn. Colors in the Figure represent cluster IDs. The names of the top-eight people are placed in their positions and they are located near the center or boundaries of the clusters. This indicates that the people are likely to have bridging roles across research communities. Researchers such as A. McCallum and E. Xing are also identified in the central positions because they have contributed to the foundation of ML/NLP while collaborating in several application domains. Surprisingly, the top \textit{Degree} nodes do not overlap with the top nodes of GRAPH-wGD. This indicates that although they may be the most prolific researchers, they have less global effect on the learned embedding, particularly if they publish only with collaborators who are placed nearby in the embedding space (i.e., less bridging role in the learned embedding). \textcolor{black}{Our results suggest that GRAPH-GD demonstrates a similar ranking to degree-based ranking, indicating that it may be more closely related to degrees rather than bridgeness-based ranking.}

    \begin{figure}[!t]    
    \centering
%    \vspace{-6mm}
    \includegraphics[width=0.60\linewidth]{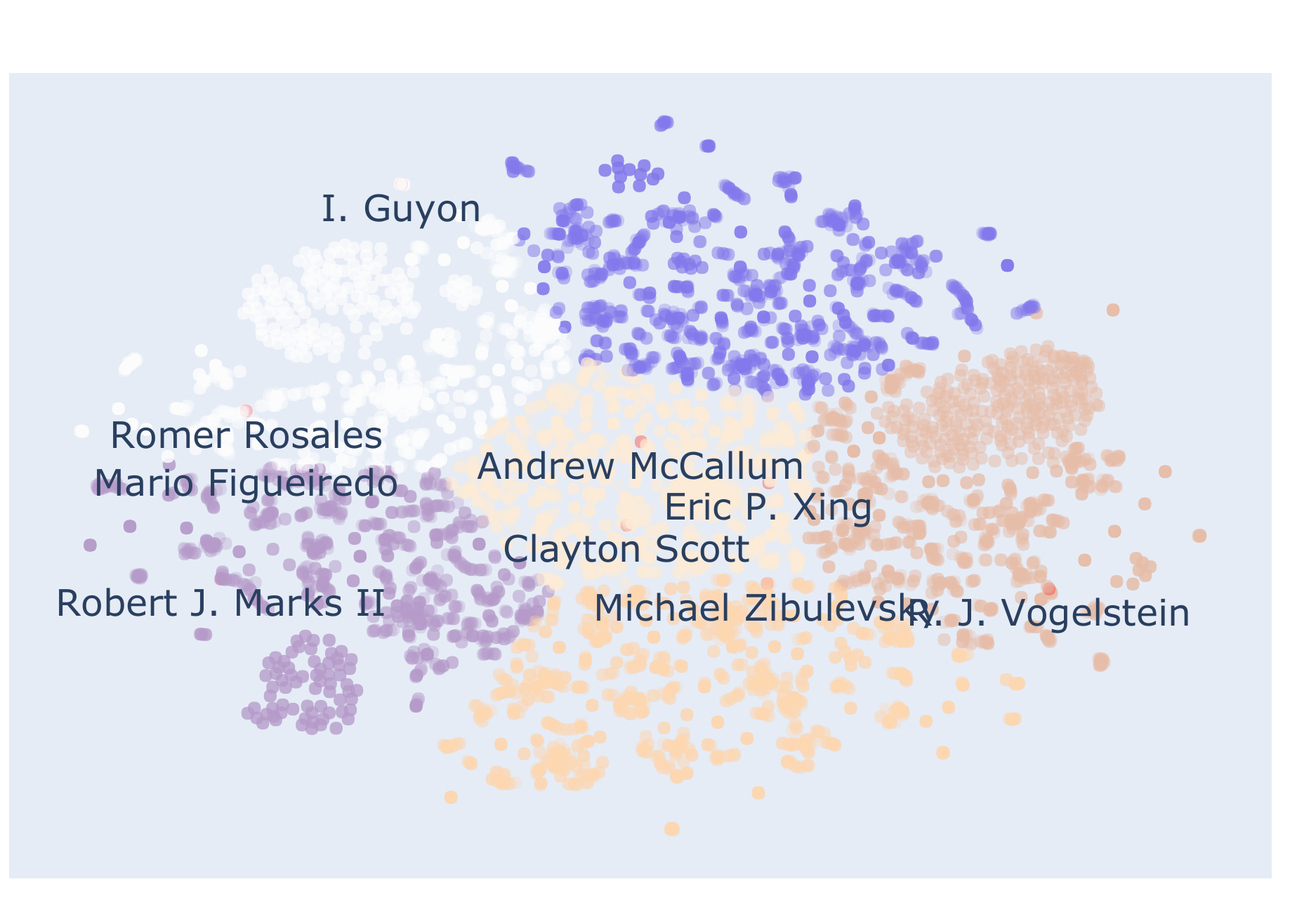}
    \caption{Visualization of the top-eight node-level explanations that are found by GRAPH-wGD for the NeurIPS dataset.}
    \label{fig:neurips-graphwgd}
    \end{figure}

\vspace{1mm}
\paragraph{Ablation Study}

In our proposed GRAPH-wGD, we might use other weighting functions by changing our assumptions w.r.t.\ support nodes. \textcolor{black}{In this section, we compare various modifications of GRAPH-wGD on the Wiki and BlogCatalog datasets, including modifications to the cosine similarity function and the inclusion of other activation functions, such as sigmoid and tanh as}:

	\resizebox{0.94\linewidth}{!}{%\
	    \begin{minipage}{\linewidth}
	    \allowdisplaybreaks{    \begin{align*}
	            &\textit{GRAPH-wGD (+ and -): } h(v_b, v_i) = 1+cos\left(\vec{w}_b -  \vec{w}_{i},  -\frac { \partial \mathcal{L}(v_b, v_i) }{ \partial \vec{w}_i}\right),\\
	            &\textit{GRAPH-wGD (abs): } h(v_b, v_i) = 1+\left| cos\left(\vec{w}_b -  \vec{w}_{i},  -\frac { \partial \mathcal{L}(v_b, v_i) }{ \partial \vec{w}_i}\right) \right|,\\
	            &\textcolor{black}{\textit{GRAPH-wGD (sigmoid): } h(v_b, v_i) = 1+\sigma\left( \left(\vec{w}_b -  \vec{w}_{i} \right) \cdot  \left(-\frac { \partial \mathcal{L}(v_b, v_i) }{ \partial \vec{w}_i} \right)\right),}\\
	            &\textcolor{black}{\textit{GRAPH-wGD (tanh): } h(v_b, v_i) = 1+tanh\left(\left(\vec{w}_b -  \vec{w}_{i}\right) \cdot  \left(-\frac { \partial \mathcal{L}(v_b, v_i) }{ \partial \vec{w}_i} \right) \right),	}	            
	        \end{align*}
	   }
	    \end{minipage}
	}

\vspace{4mm}
\noindent where \textit{cos} is a cosine similarity function. Here, GRAPH-wGD (+ and -) allows negative angular distance to offset the importance when aggregating scores, and GRAPH-wGD (abs) denotes that negative and positive support nodes are handled in the same way. 
\textcolor{black}{GRAPH-wGD (sigmoid) and GRAPH-wGD (tanh) use sigmoid and tanh functions to calculate the distance between two vectors by computing their dot product values. These functions limit the range of their outputs to [0,1] and [-1,1] respectively. Compared to GRAPH-wGD (abs), GRAPH-wGD (sigmoid) converts negative inputs into positive values but reduces their magnitude GRAPH-wGD (tanh) is similar to GRAPH-wGD (+ and -) but emphasizes the magnitude of negative and positive values close to zero.}
Table \ref{tab:wikiblog:ablation:corr} indicates that our choice of GRAPH-wGD is better than GRAPH-GD and all other variants in Spearman’s rank correlation. This observation also indicates that identifying support nodes improves the identification of bridge nodes.
\textcolor{black}{
In particular, GRAPH-wGD (sigmoid) enhances the performance compared to GRAPH-wGD (abs), however, the performance is still not as good as GRAPH-wGD. 
Similarly, GRAPH-wGD (tanh) also shows worse results, and the results of GRAPH-wGD with two activation functions indicate that the cosine similarity-based measure is effective in capturing bridging roles through adjusting distances.
GRAPH-wGD ([-30\textdegree, 30\textdegree]) only uses the cosine distance when the angular distance between two vectors falls between -30\textdegree and 30\textdegree. This selectively identifies nodes with stronger bridging roles and assigns zeros to others. The degree of filtering can be controlled by the range values [-30\textdegree, 30\textdegree], [-45\textdegree, 45\textdegree], and [-60\textdegree, 60\textdegree]. Our experiment shows that GRAPH-wGD ([-30\textdegree, 30\textdegree]) significantly improves GRAPH-GD. As the range of consideration widens, the performance also improves, suggesting that incorporating the angular distance is effective in identifying bridgeness-based important nodes.}

\begin{table}[t]
% 	\vspace{-4mm}
    \begin{minipage}[h]{.95\textwidth }
		\centering
		 \caption{Ablation Study: Spearman's Rank Correlation} 
		\label{tab:wikiblog:ablation:corr}     		
		\scalebox{0.85}{      
		\begin{tabular}{c|cc|cc}
           & \multicolumn{2}{c|}{Wiki}       & \multicolumn{2}{c}{BlogCatalog} \\
           & DeepWalk       & LINE           & DeepWalk       & LINE           \\ \hline
            GRAPH-GD & 0.197 & 0.158 & 0.268 & 0.141 \\ \hline
            \textbf{GRAPH-wGD (+ and -)} & 0.292 & 0.122 & 0.195 & 0.181 \\ 
            \textbf{GRAPH-wGD (abs)} & 0.189 & 0.197 & 0.122 & 0.225 \\ 
            \textbf{\textcolor{black}{GRAPH-wGD (sigmoid)}} & \textcolor{black}{0.214} & \textcolor{black}{0.225} & \textcolor{black}{0.145} & \textcolor{black}{0.238} \\ 
            \textbf{\textcolor{black}{GRAPH-wGD (tanh)}} & \textcolor{black}{0.302} & \textcolor{black}{0.142} & \textcolor{black}{0.214} & \textcolor{black}{0.203} \\             \hline                       
            \textbf{\textcolor{black}{GRAPH-wGD ([-30\textdegree{}, 30\textdegree{}])}} & \textcolor{black}{0.313} & \textcolor{black}{0.387} & \textcolor{black}{0.307} & \textcolor{black}{0       .294} \\            
            \textbf{\textcolor{black}{GRAPH-wGD ([-45\textdegree{}, 45\textdegree{}])}} & \textcolor{black}{0.334} & \textcolor{black}{0.412} & \textcolor{black}{0.335} & \textcolor{black}{0.303}	\\
            \textbf{\textcolor{black}{GRAPH-wGD ([-60\textdegree{}, 60\textdegree{}])}} & \textcolor{black}{0.362} & \textcolor{black}{0.435} & \textcolor{black}{0.350} & \textcolor{black}{0.321}	\\\hline
            \textbf{GRAPH-wGD} & \textbf{0.392} & \textbf{0.458} & \textbf{0.368} & \textbf{0.334}		
		\end{tabular}
		}
		\hrule height 0pt
    \end{minipage}% 
\end{table}

\begin{table}[t]
        % \vspace{-4mm}
    \begin{minipage}[h]{.95\textwidth }
        \centering
    \caption{Elapsed time in Wiki and BlogCatalog} 
    \label{tab:wikiblog:elapsedtime}        
        \scalebox{0.85}{             
            \begin{tabular}{c|ll|ll}
                   & \multicolumn{2}{c|}{Wiki}       & \multicolumn{2}{c}{BlogCatalog} \\
                   & DeepWalk       & LINE           & DeepWalk       & LINE           \\ \hline
                    Bridgeness           & 748.41 sec  &  748.41 sec   &   10,039.59 sec   &   10,039.59 sec    \\ \hline
                    Greedy             & 214.32 hours  &  28.72 hours   &   902.24 hours   &  68.57 hours    \\
                    LIME               & 156.69 days & 12.65 days    &   582.56 days    & 381.03 days        \\
                    GNNExplainer       & 192.20 hours &  29.78 hours   &  850.03 hours     &     70.49   hours \\ 
                    GENI             & 822.97 sec  &  827.47 sec   &   10,345.21 sec   &  10,421.32 sec    \\ \hline
                    \textbf{GRAPH-wGD} & \textbf{23.30 sec} &   \textbf{67.41 sec}   &  \textbf{141.93 sec}     & \textbf{233.44 sec}
            \end{tabular}
            }
        \hrule height 0pt
    \end{minipage}%     
\end{table}

%\newpage
\subsection{Runtime Comparison} \label{sec:runtime}

In Table \ref{tab:wikiblog:elapsedtime}, we report the runtime of our method on Wiki and BlogCatalog datasets. Our method was implemented in Python and was executed on an Intel Xeon Gold 6126 CPU@2.60~GHz server with 192~GB RAM. The reported running time was measured by assuming that the code was run by a single process. All the methods were executed by multiple processes and merged in the actual computations. As in Table \ref{tab:wikiblog:elapsedtime}, our algorithm scales considerably compared with other alternatives. In particular, \textit{Greedy}, LIME, and GNNExplainer had to re-learn node embedding every time they need a perturbation. Thus, in the case of DeepWalk with the Skip-gram architecture, it takes $O(|V|\log|V|)$, whereas \textit{Greedy} and GNNExplainer take $O(|V|^2 \log|V|)$ owing to the node-wise perturbations.  Similarly, LIME takes $O(|V|^3 \log|V|)$ to consider additional local neighbor perturbations. For \textit{Bridgeness}, we run both embedding and EVD, which takes $O(|V|^{2.367})$. EVD can be replaced by \citep{garber2016faster} to obtain $O(nnz(A))$, where $nnz(A)$ denotes the number of non-zeros in $A$, but when $A$ is dense, $O(nnz(A))=O({|V|}^2)$. GENI also corresponds to $O(|E|)$ in addition to its bridgeness computation. Meanwhile, our method does not need to learn the graph again and takes only $O(|V|)$.

%!TEX root = main.tex

% \newpage
\section{Conclusion}\label{sec:conclusion}

We demonstrated the theoretical relationship between the \textit{node-level explanation} of node embeddings including DeepWalk and LINE, and \textit{bridgeness}. Two new algorithms, GRAPH-GD and GRAPH-wGD, were also proposed for generating \textit{post-hoc} explanations to the learned embeddings more efficiently. In particular, our GRAPH-wGD exhibited superior performance w.r.t.\ three evaluation measures over other alternatives on five different datasets.

%\section*{Reference}

%\bibliographystyle{named} 
\bibliographystyle{elsarticle-num}
\bibliography{reference}
%%\bibliographystyle{aaai}
%
%!TEX root = main.tex

\clearpage
\appendix

\section{Proofs} \label{sec:proofs}
% \sectionfont{\small}

Let $G$ be an input graph with vertices $V$ and let $A$ be the adjacency matrix of $G$. We note that the initial set of clusters is given before perturbation. We use spectral clustering to obtain the set of clusters, $\{C_1$, ..., $C_k\}$, for ease of theoretical analysis for theorems \ref{theorem:eigenvec:bridge} and \ref{theorem:n2v:bridge}. We note that other lemmas and theorems work on arbitrary clusters. There are several objective functions that capture the clusters and we focus on finding eigenvectors to maximize the relaxed normalized association $N_{asso}(C_i|G)= \frac{assoc(C_i, C_i|G)}{assoc(C_i, V|G)}$ \citep{niu2011dimensionality}, which are from the given graph $G$. Here, $assoc(C_i, V|G)$ represents the total connections between $C_i$ and all nodes and we set $assoc(C_i, V|G)=\sum_{s\in C_i, t\in V} A[s,t]$. Similarly, $assoc(C_i, C_i|G)=\sum_{s\in C_i, t\in C_i} A[s,t]$.

Using the Perturb function, we can analyze how $N_{asso}(C_i|G)$ will change after the perturbation. Again, $N_{asso}(C_i|G)= \frac{assoc(C_i, C_i|G)}{assoc(C_i, V|G)}$, which are from the given graph $G$ and its (arbitrary) clusters $C$. 

\begin{lemma}
Given a graph $G$, its set of (arbitrary) clusters $C=\{C_1, ..., C_k\}$ and an arbitrary node $v_b \in V$, let $G'_{-v_b}=$ Perturb($G$, $v_b$, $\alpha=1$) (see  definition \ref{def:perturb} in \ref{sec:proofs}). 
Then $N_{asso}(C_i|G) \leq N_{asso}(C_i|{G'}_{-v_b})$ for all $C_i \subset C$.

\label{lemma:perturb}
\end{lemma}

\begin{proof}[Proof of Lemma \ref{lemma:perturb}]
    By definition, $N_{asso}$ with and without perturbation are written as
    \vspace{-2mm}
    \begin{equation*}
        \resizebox{0.49\hsize}{!}{%    
		    $N_{asso}(C_i|G)=\frac{assoc(C_i, C_i|G)}{assoc(C_i, V|G)}, N_{asso}(C_i|G'_{-v_b})=\frac{assoc(C_i, C_i|G'_{-v_b})}{assoc(C_i, V|G'_{-v_b})}.$
		}		    
    \end{equation*}
\vspace{-2mm}
      
\noindent The degree-preserving perturbation does not change the degrees of any vertices, so the denominator does not change (i.e., $assoc(C_i, V|G)=assoc(C_i, V|G'_{-v_b})$).
However, the perturbation can increase the number of edges in $C_i$ (by removing edges to other clusters and replacing them with self-edges), thus the numerator may increase (i.e., $assoc(C_i, C_i|G) \leq assoc(C_i, C_i|G'_{-v_b})$). Therefore, $N_{asso}(C_i|G) \leq N_{asso}(C_i|{G'}_{-v_b})$.
\end{proof}

The following lemma shows that the difference in $N_{asso}$ is maximized after perturbing the node with the largest \textit{bridgeness}. For $N_{asso}$ for multi-clusters, $N_{asso}(C_1, ..., C_k|G)= \sum_{i=1}^k N_{asso}(C_i|G) = \sum_{i=1}^k {\frac{assoc(C_i, C_i|G)}{assoc(C_i, V|G)}}$ 
 
\begin{lemma}
    Let $V$ be the nodes in graph $G$ with (arbitrary) clusters $C=\{C_1, ..., C_k\}$ and $G'_{-v_j}=$ Perturb($G$, $v_j$, $\alpha=1$) (see  definition \ref{def:perturb} in \ref{sec:proofs}). Assume that $\gamma(C_i)=1$ for all $C_i \subset C$ to compute \textit{Bridgeness} in definition \ref{def:bridgenode:bridgeness}. If we define
    \vspace{-2mm}
    \begin{equation*}
        \resizebox{0.51\hsize}{!}{%
            $v_b = \text{arg}\,\max\limits_{v_j \in V}\, (N_{asso}(C_1, ..., C_k | G'_{-v_j}) - N_{asso}(C_1, ..., C_k | G)),$
        }
    \end{equation*}
    \vspace{-2mm}
    \noindent then $v_b$ is the node with highest \textit{Bridgeness}. 
\label{lemma:argmaxdist}
\end{lemma}
\vspace{-2mm}

\begin{proof}
By definition, $N_{asso}(C_i|G)=\frac{assoc(C_i, C_i|G)}{assoc(C_i, V|G)}$. Note that in Lemma \ref{lemma:perturb}, when $v_b \in C_i$ has at least one inter-cluster edge and the edges of $v_b$ are perturbed, $N_{asso}(C_i|G'_{-v_b})$ is larger than $N_{asso}(C_i|G)$. The difference is exactly the number of inter-cluster edges, which are all perturbed and replaced with self-edges. Therefore, the node with the largest inter-cluster degree maximizes this difference. Assuming $\gamma(C_i)=1$ in \textit{Bridgeness}($v_b\in C_i|C, G$) (see definition \ref{def:bridgenode:bridgeness}), the node $v_b$ with maximum \textit{Bridgeness} score has the highest number of inter-cluster edges. As a result, $v_b$ also maximizes $N_{asso}(C_i|G'_{-v_b}) - N_{asso}(C_i|G)$. 

\vspace{-4mm}

\end{proof}

\vspace{-2mm}

Now we can derive the node which is a \textit{node-level explanation} w.r.t.\ a spectral embedding (eigenvector $U$ from $D^{-1/2}AD^{-1/2}$), where $D$ is a degree matrix of $A$.

\vspace{2mm}
\noindent \textbf{Theorem \ref{theorem:eigenvec:bridge}.}
Let $\{C_1, ..., C_k\}$  be $k$ clusters from spectral clustering of graph $G$, with nodes $V$. Let $G'_{-v_j}$ be a perturbed version of $G$ using Perturb($G$, $v_j$, $\alpha\!=\!1$) (definition \ref{def:perturb}). Let  $U \in \mathbb{R}^{|V| \times k}$ and $U'_{-v_j} \in \mathbb{R}^{|V| \times k}$ be the eigenvectors of $G$ and ${G'}_{-v_j}$, respectively. Then, \\
\vspace{-2mm}
\begin{equation*}
\resizebox{0.55\hsize}{!}{%
    $\text{arg}\!\,\max\limits_{v_j \in V} \widehat{Imp}(v_j|G, {G'}_{-v_j}, U,{U'}_{-v_j}) = \text{arg}\!\,\max\limits_{v_j \in V} Bridgeness(v_j|C, G)$.
}
\end{equation*}

\begin{proof}
By the definition of NI ($Imp$) of $v_j$ using pairwise distances ($PD$) as in definition \ref{def:pairwisenodeimp}, we can see the relationship between $Imp$ and $N_{asso}$:\\

\resizebox{0.94\linewidth}{!}{%\
\begin{minipage}{\linewidth}
	\begin{alignat}{1}
	\ & \widehat{Imp}(v_j|G, G'_{-v_j}, U, U'_{-v_j})  = || PD(G, U) - PD(G'_{-v_j}, U'_{-v_j}) || \label{thm1:line2}\\	
	= \ & ||  2\sum_{m=1}^k(I) - 2\sum_{m=1}^k(I)   +   PD(G, U)  - PD(G'_{-v_j}, U'_{-v_j}) || \label{thm1:line4}\\
	= \ & || (2(\sum_{m=1}^k(I) - \frac{1}{2}PD(G'_{-v_j}, U'_{-v_j})) -2(\sum_{m=1}^k(I) - \frac{1}{2}PD(G, U)) || \label{thm1:line6}\\
	= \ & || (2(\sum_{m=1}^k(I) - \frac{1}{2} \sum_{v_i, v_j \in V}  \left\| \frac{U'_{-v_j}[i, 1:k]}{\sqrt{d_i}} - \frac{U'_{-v_j} [j, 1:k]}{\sqrt{d_j}} \right\|^2 ) \label{thm1:line7}\\
	 & \ \ \ \ \ \ \  -2(\sum_{m=1}^k(I) - \frac{1}{2} \sum_{v_i, v_j \in V}  \left\| \frac{U[i, 1:k]}{\sqrt{d_i}} - \frac{U[j, 1:k]}{\sqrt{d_j}} \right\|^2 ) || \label{thm1:line8}\\
	= \ & || (2(\sum_{m=1}^k(I) - \sum_{m=1}^k {U'}^T_{-v_j}[:,m] {L'}_{sym} {U'_{-v_j}}[:,m] ) \label{thm1:line9}\\
	 & \ \ \ \ \ \ \  -2(\sum_{m=1}^k(I) - \sum_{m=1}^k U^T[:,m] L_{sym} U[:,m] ) || \label{thm1:line10}\\
	= \ & 2(\sum_{m=1}^k  {U'}_{-v_j}^T[:,m] {A'}_{sym} {U'}_{-v_j}[:,m]  - \sum_{m=1}^k  U^T[:,m] A_{sym} U[:,m]) \label{thm1:line12}\\
	= \ & 2(N_{asso}(C_1, ..., C_k | G'_{-v_j}) - N_{asso}(C_1, ..., C_k | G)) \label{thm1:line13}
	\end{alignat}
\end{minipage}
}

\vspace{2mm}
	   
\noindent In these equations, lines (\ref{thm1:line2})--(\ref{thm1:line8}) are developed by directly using definition \ref{def:pairwisenodeimp} in the main paper with basic arithmetic operations. Then, lines (\ref{thm1:line9})--(\ref{thm1:line10}) are developed using the following equation as in \citep{von2007tutorial}:
\begin{equation*}
\resizebox{0.45\hsize}{!}{%
	$\frac{1}{2} \sum_{v_i, v_j \in V} \left\| \frac{U_{i, 1:k}}{\sqrt{d_i}} - \frac{U_{j, 1:k}}{\sqrt{d_j}} \right\|^2  = \sum_{m=1}^k U^T_{:,m} L_{sym} U_{:,m},$
}	
\end{equation*}
where $L_{sym}=D^{-1/2}LD^{-1/2}$. We note that $L=D-A$. Line (\ref{thm1:line12}) is also from $A_{sym} = I - L_{sym}$. The last line (\ref{thm1:line13}) is from the objective function of spectral clustering, which finds $U$ for maximizing $\sum_{m=1}^k  U^T[:,m] A_{sym} U[:,m]$ s.t. $U^T U =I$, and it is equivalent to maximizing $N_{asso}(C_1, ..., C_k | G)$ \citep{niu2011dimensionality}. Now we expect that the node importance of $v_j$ corresponds to the difference in $N_{asso}$. 

From Lemma \ref{lemma:argmaxdist}, the difference in $N_{asso}$ under perturbation is maximized by the node that has the highest \textit{bridgeness}. Therefore, the highest-\textit{bridgeness} node has the highest NI.
\end{proof}

Similarly, here we consider the DeepWalk embedding to identify the node having the highest \textit{Importance} (i.e., node-level explanation) under perturbation. Let $A$ be the adjacency matrix of $G$. We use $D_{ii} = \sum_j A_{ij}$ to denote the diagonal degree matrix.

\vspace{4mm}
\noindent \textbf{Theorem \ref{theorem:n2v:bridge}.}
Let $G'_{-v_j}$ represent a perturbed version of graph $G$ with $V$ nodes, using Perturb($G$, $v_j$, $\alpha\!=\!1$) (definition \ref{def:perturb}). Let $W \in \mathbb{R}^{|V| \times k}$ and ${W'}_{-v_j} \in \mathbb{R}^{|V| \times k}$ be embeddings from DeepWalk (with context window size $r$) over $G$ and ${G'}_{-v_j}$, respectively. 
Let $\tilde{C}=\{\tilde{C}_1, ..., \tilde{C}_k\}$ be $k$ clusters from spectral clustering of $\tilde{G}_r$, where $\tilde{G}_r$ is an $r^{th}$ weighted power transformation of $G$ 
that is degree-normalized. 
Then,\\
\vspace{-2mm}
\begin{equation*}
\resizebox{0.59\hsize}{!}{%
    $\text{arg}\!\,\max\limits_{v_j \in V} \widehat{Imp}(v_j|G, {G'}_{-v_j}, W,{W'}_{-v_j}) = \text{arg}\!\,\max\limits_{v_j \in V} Bridgeness(v_j| \tilde{C}, \tilde{G}_r)$.
}
\end{equation*}

\begin{proof} In \citep{qiu2018network}, the DeepWalk could be approximated by the matrix factorization of $\tilde{G}_r$ and $W$ denotes its eigenvectors which correspond to the top-$k$ eigenvalues. 
Therefore, using Theorem \ref{theorem:eigenvec:bridge} with $W$, ${W'}_{-v_j}$, $\tilde{C}_r$, and $\tilde{G}_r$, we can see the relationship between \textit{NI} on DeepWalk embedding and \textit{Bridgeness} as follows:

\begin{equation*}
\resizebox{0.60\hsize}{!}{%
	    $\text{arg}\,\max\limits_{v_j \in V} Bridgeness(v_j|\tilde{C}, \tilde{G}_r)
        =     \text{arg}\,\max\limits_{v_j \in V} \widehat{Imp}(v_j|\tilde{G}, {\tilde{G}^{\prime}}_{-v_j}, W,{W'}_{-v_j}).$
}
\end{equation*}

\noindent In this equation, according to the definition of NI $\widehat{Imp}$ (see definition \ref{def:pairwisenodeimp} in the main paper) takes two embeddings ($W, W'_{-v_j}$), a degree matrix $D$, and a vertex set $V$. Although another graph input is given as $\tilde{G}_r$, we still have the same vertex set $V$. The degree $\tilde{D}_r$ from  $\tilde{G}_r$ also corresponds to the original degree matrix $D$ when $r$ is sufficiently large (i.e., $\tilde{D}_r = D/(2|E|)) \approx D$ \citep{spitzer2013principles}), where $|E|$ is the number of edges in $G$. Therefore, when the same embeddings ($W, W'_{-v_j}$) are provided to the new importance function with $\tilde{G}_r$, we can obtain the following equation:\\

\vspace{-4mm}
\begin{equation*}
\resizebox{0.62\hsize}{!}{%
	    $\text{arg}\,\max\limits_{v_j \in V} \widehat{Imp}(v_j|\tilde{G}, {\tilde{G}^{\prime}}_{-v_j}, W,{W'}_{-v_j}) = \text{arg}\,\max\limits_{v_j \in V} \widehat{Imp}(v_j|G, {G}^{\prime}_{-v_j}, W,{W'}_{-v_j}).$
}
\end{equation*}
\vspace{-4mm} 
\end{proof}

\noindent Based on Theorem \ref{theorem:n2v:bridge}, we can outline how the same node is also likely to have the highest bridgeness score in \\$G$ (i.e.,   $\text{arg}\!\,\max\limits_{v_j \in V} Bridgeness(v_j| \tilde{C}, \tilde{G}_r) \approxeq \text{arg}\!\,\max\limits_{v_j \in V} Bridgeness(v_j| C, G)$) as in Lemma \ref{lemma:n2v:bridge_connect1}, Lemma \ref{lemma:n2v:bridge_connect2}, and Corollary \ref{cor:n2v:bridge}.

\begin{lemma}
Let $G'_{-v_j}$ represent a perturbed version of graph $G$ with $V$ nodes, using Perturb($G$, $v_j$, $\alpha\!=\!1$) (see  definition \ref{def:perturb} in \ref{sec:proofs}). Let $W \in \mathbb{R}^{|V| \times k}$ and ${W'}_{-v_j} \in \mathbb{R}^{|V| \times k}$ be embeddings from DeepWalk (with context window size $r$) over $G$ and ${G'}_{-v_j}$, respectively. 
Let ${C}=\{{C}_1, ..., {C}_k\}$ be $k$ clusters from spectral clustering of $G$. We also have $\tilde{C}=\{\tilde{C}_1, ..., \tilde{C}_k\}$ from spectral clustering of $\tilde{G}_r$, where $\tilde{G}_r$ is an $r^{th}$ weighted power transformation of $G$ 
that is degree-normalized. 
Then,\\
\vspace{-4mm}
\begin{equation*}
\resizebox{0.57\hsize}{!}{%
    $ \text{arg}\!\,\max\limits_{v_j \in V} Bridgeness(v_j| \tilde{C}, \tilde{G}_r) \approxeq \text{arg}\!\,\max\limits_{v_j \in V} Bridgeness(v_j| C, \tilde{G}_r)$.
}
\end{equation*}
\vspace{-4mm}
\label{lemma:n2v:bridge_connect1}
\end{lemma}

\begin{proof}
\noindent First, adjacency matrices of $G$ and $\tilde{G}_r$ are given as $A_{sym}$ and $\tilde{A}_{sym}^r= \left(\sum_{t=1}^r\nolimits (D^{-1}A)^{t}\right)D^{-1}$, respectively \citep{qiu2018network}. Again, $D_{ii} = \sum_j A_{ij}$. Then, we can find their relationship between $\tilde{A}_{sym}^r$ and ${A}_{sym}$ as follows: 

    \vspace{-2mm}

    \resizebox{0.97\linewidth}{!}{%\
\begin{minipage}{\linewidth}
	\begin{alignat}{1}
	\tilde{A}^r_{sym} &= \left(\sum_{t=1}^r\nolimits (D^{-1}A)^{t}\right)D^{-1}  \nonumber \\ 
	&=D^{-1/2} \sum_{t=1}^r\nolimits (D^{-1/2}A D^{-1/2})^{t} D^{-1/2} \label{eq:twoasyms}    \\
        &=D^{-1/2} \sum_{t=1}^r\nolimits (A_{sym})^{t} D^{-1/2}. \nonumber
	\end{alignat}
\end{minipage}
}

\vspace{1mm}

\noindent In particular, the relationship between $A_{sym}$ and $\sum_{t=1}^t (A_{sym})^{t}$ w.r.t. eigendecompositions could be identified as follows:

    \resizebox{1.00\linewidth}{!}{%\
\begin{minipage}{\linewidth}
	\begin{alignat}{1}
	&A_{sym} = U  \Lambda U^T \nonumber \\
	&\sum_{t=1}^t (A_{sym})^{t} = U (\sum_{r=1}^T \Lambda^{r}) U^T. \label{eq:eigen:asym}	
	\end{alignat}
\end{minipage}
}
\vspace{1mm}

\noindent In Equation (\ref{eq:eigen:asym}), the eigenvectors of $A_{sym}$ and $\sum_{t=1}^t (A_{sym})^{t}$ are equivalent. We note that the spectral clusters are found by the eigenvectors of adjacency matrix with a discrete partitioning using a rounding step \citep{von2007tutorial}. For example, the rounding step is composed of the normalization of eigenvectors and apply $k$-means \citep{ng2002spectral}. When we replace $\sum_{t=1}^t (A_{sym})^{t}$ of Equation (\ref{eq:twoasyms}) by its eigendecomposition in Equation
(\ref{eq:eigen:asym}), we obtain

\begin{equation}
\vspace{-8mm}
\resizebox{0.45\hsize}{!}{%
	    $\tilde{A}^r_{sym} = D^{-1/2} U (\sum_{r=1}^T \Lambda^{r}) U^T D^{-1/2} = \bar{U} (\sum_{r=1}^T \Lambda^{r}) \bar{U}^{T}.$
}
	\label{eq:eigen:approximation}
\end{equation}
\vspace{6mm}

\noindent Using this observation, we can see that $\bar{U}$ is the normalized eigenvector matrix, which could be the part of the rounding step. When it leverages the same normalization, its final discrete partitions are not likely to be different from spectral clusters from $\tilde{A}_{sym}$ and $\sum_{t=1}^t (A_{sym})^{t}$. Therefore, we can obtain $\tilde{C} \approxeq C$ in spectral clustering, which implies that $ \text{arg}\!\,\max\limits_{v_j \in V} Bridgeness(v_j| \tilde{C}, \tilde{G}_r) \approxeq \text{arg}\!\,\max\limits_{v_j \in V} Bridgeness(v_j| C, \tilde{G}_r)$. 
\end{proof}

\clearpage

\begin{lemma}
Let $G'_{-v_j}$ be a perturbed version of graph $G$ with $V$ nodes, using Perturb($G$, $v_j$, $\alpha\!=\!1$) (see  definition \ref{def:perturb} in \ref{sec:proofs}). Here $W \in \mathbb{R}^{|V| \times k}$ and ${W'}_{-v_j} \in \mathbb{R}^{|V| \times k}$ are embeddings from DeepWalk (with context window size $r$) over $G$ and ${G'}_{-v_j}$, respectively. 
Let ${C}=\{{C}_1, ..., {C}_k\}$ be $k$ clusters from spectral clustering of $G$. We also have $\tilde{C}=\{\tilde{C}_1, ..., \tilde{C}_k\}$ from spectral clustering of $\tilde{G}_r$, where $\tilde{G}_r$ is an $r^{th}$ weighted power transformation of $G$ that is degree-normalized. Then,

\vspace{-4mm}
\begin{equation*}
\resizebox{0.55\hsize}{!}{%
    $ \text{arg}\!\,\max\limits_{v_j \in V} Bridgeness(v_j| C, \tilde{G}_r) \approxeq \text{arg}\!\,\max\limits_{v_j \in V} Bridgeness(v_j| C, {G}_r)$.
}
\end{equation*}
\label{lemma:n2v:bridge_connect2}
\end{lemma}

\vspace{-6mm}
\begin{proof}
We expect that the inter-cluster distances of clusters in $\tilde{G}_r$ are larger than those in $G$ \citep{yen2005clustering,pons2006computing}. Meanwhile, when their clusters are the same, they are likely to have similar inter-cluster distances and their rankings in inter-cluster degrees of $G$ and $\tilde{G}_r$ also become similar, which means $ \text{arg}\!\,\max\limits_{v_j \in V} Bridgeness(v_j| C, \tilde{G}) \approxeq \text{arg}\!\,\max\limits_{v_j \in V} Bridgeness(v_j| C, G)$.	
\end{proof}

\begin{corollary}
Let $G'_{-v_j}$ be a perturbed version of $G$ with $V$ nodes, using Perturb($G$, $v_j$, $\alpha\!=\!1$) (see  definition \ref{def:perturb} in \ref{sec:proofs}). Let $W \in \mathbb{R}^{|V| \times k}$ and ${W'}_{-v_j} \in \mathbb{R}^{|V| \times k}$ be embeddings from DeepWalk (with context window size $r$) over $G$ and ${G'}_{-v_j}$, respectively. 
Let ${C}=\{{C}_1, ..., {C}_k\}$ be $k$ clusters from spectral clustering of $G$. We also have $\tilde{C}=\{\tilde{C}_1, ..., \tilde{C}_k\}$ from spectral clustering of $\tilde{G}_r$, where $\tilde{G}_r$ is an $r^{th}$ weighted power transformation of $G$ 
that is degree-normalized. 
Then,
\vspace{-2mm}
\begin{equation*}
\resizebox{0.55\hsize}{!}{%
    $\text{arg}\!\,\max\limits_{v_j \in V} Bridgeness(v_j| \tilde{C}, \tilde{G}_r) \approxeq \text{arg}\!\,\max\limits_{v_j \in V} Bridgeness(v_j| C, G)$.
}
\end{equation*}
\vspace{-2mm}
\label{cor:n2v:bridge}
\end{corollary}
\vspace{-6mm}

\begin{proof}
Using Theorem \ref{theorem:n2v:bridge}, Lemma \ref{lemma:n2v:bridge_connect1}, and Lemma \ref{lemma:n2v:bridge_connect2}, the highest-bridgeness node with $\tilde{C}_r$ and $\tilde{G}_r$ equals the highest-bridgeness node with $C$ and $G$ as follows:
\vspace{-2mm}
\begin{equation*}
\resizebox{0.55\hsize}{!}{%
    $\text{arg}\!\,\max\limits_{v_j \in V} Bridgeness(v_j| \tilde{C}, \tilde{G}_r) \approxeq \text{arg}\!\,\max\limits_{v_j \in V} Bridgeness(v_j| C, G)$.
}
\end{equation*}

\end{proof}

Using Theorem \ref{theorem:n2v:bridge} and Corollary \ref{cor:n2v:bridge}, we can present the relationship between NI on DeepWalk and \textit{Bridgeness} as
%\vspace{-4mm}

\begin{equation*}
\resizebox{0.59\hsize}{!}{%
	$\text{arg}\,\max\limits_{v_j \in V} \widehat{Imp}(v_j|G, {{G}^{\prime}}_{-v_j}, W,{W'}_{-v_j}) 
        \approxeq \ \text{arg}\,\max\limits_{v_j \in V} Bridgeness(v_j|C, G).$
}	
%\label{prof:thm:deepwalk:bridge}
\end{equation*}

\noindent Now we provide proofs to show the relationship of GRAPH-GD and GRAPH-wGD to \textit{Bridgeness}. First, based on the definition \ref{def:goodembedding}, the following analysis shows that GRAPH-GD enables the identification of bridge nodes $v_b$, compared with any node $v_c$ that has the same degree but lower \textit{bridgeness}.  Notation ($W$, $\mathcal{N}_{G}$, $\mathcal{L}_{\mathcal{M}}$) in GRAPH-GD does not change in this analysis, so they are dropped in the following. We assume that all neighbors are returned from $\mathcal{N}_{G}(\cdot|\psi)$ in the following analysis of GRAPH-GD and GRAPH-wGD and we call $\mathcal{N}_G(\cdot)$.

\vspace{2mm}
\noindent \textbf{Lemma \ref{lemma:graphgd:bridge}.} 
Let $G$ be a graph with a set of (arbitrary) clusters $C$ and associated embedding $W$. Let $v_b \in C_b$ be a bridge node and $v_c \in C_b$ be a node with the same degree as $v_b$, but fewer inter-cluster edges (thus, lower \textit{bridgeness}). Here $\text{B-Cluster}(v_b)$ denotes the set of clusters that are bridged by $v_b$. Assume that $\text{B-Cluster}(v_b) \supseteq \text{B-Cluster}(v_c)$ and that $W$ is a {\em good embedding} as in Definition \ref{def:goodembedding}. Then, GRAPH-GD($v_b$) $>$ GRAPH-GD($v_c$).

\vspace{-2mm}
\begin{proof}
First, we give the definition of GRAPH-GD again and GRAPH-GD is presented by using the distance function $dist$. Here, $dist$ is used for measuring the distance between two nodes on the embedding space. By assuming that the distance corresponds to the gradient update of $\left| \frac { \partial \mathcal{L}_{\mathcal{M}}(v_b, v_i) }{ \partial \vec{w}_i}\right|$, we can reformulate the GRAPH-GD of $v_b$ as follows:
%\vspace{-4mm}

	\begin{equation*}
		\resizebox{0.72\hsize}{!}{%
        $\text{GRAPH-GD}(v_b) = \text{MEAN}\left\{\left| \frac { \partial \mathcal{L}_{\mathcal{M}}(v_b, v_i) }{ \partial \vec{w}_i}\right|, v_i \in \mathcal{N}_G(v_b)\right\} = \text{MEAN}\{dist(\vec{w}_b, \vec{w}_i), v_i \in \mathcal{N}_G(v_b))\}.$
		}	
    \end{equation*}

\noindent From the definition of GRAPH-GD, we can decompose GRAPH-GD($v_b$) using cluster memberships as follows:

\vspace{-1mm}
    
\resizebox{0.95\linewidth}{!}{%\
\begin{minipage}{\linewidth}
	\begin{alignat}{1}
	&\text{GRAPH-GD}(v_b) =  \text{MEAN}\{dist(\vec{w}_b, \vec{w}_i), v_i \in \mathcal{N}_G(v_b)\}  \nonumber \\
        &=\text{MEAN}(\{dist(\vec{w}_b, \vec{w}_i), v_i \in \mathcal{N}_G(v_b) \cap  C_b \} + \nonumber \\ 
        &\{dist(\vec{w}_b, \vec{w}_k), v_k \in \mathcal{N}_G(v_b) \cap  C/C_b \}) \nonumber
	\end{alignat}
\end{minipage}
}

\vspace{4mm}

\noindent Meanwhile, GRAPH-GD($v_c$) is also described as
    \vspace{-1mm}
    
\resizebox{0.95\linewidth}{!}{%\
\begin{minipage}{\linewidth}
	\begin{alignat}{1}
	&\text{GRAPH-GD}(v_c) =  \text{MEAN}\{dist(\vec{w}_c, \vec{w}_i), v_i \in \mathcal{N}_G(v_c)\}  \nonumber \\
        &=\text{MEAN}(\{dist(\vec{w}_c,  \vec{w}_i), v_i \in \mathcal{N}_G(v_c) \cap  C_b \} + \nonumber \\ 
        &\{dist(\vec{w}_c, \vec{w}_k), v_k \in \mathcal{N}_G(v_c) \cap  C/C_b \}) \nonumber
	\end{alignat}
\end{minipage}
}
\vspace{3mm}

\noindent In these equations, $v_b$ has more inter-cluster edges and $|\{\mathcal{N}_G(v_b) \cap  C/C_b\}| $ $>$ $|\{\mathcal{N}_G(v_c) \cap  C/C_b\}|$. Under the \textit{good embedding} assumption, we can expect larger distance values for the inter-cluster edges in $W$. Therefore, when the degrees of $v_b$ and $v_c$ are the same and $\text{B-Cluster}(v_b) \supseteq \text{B-Cluster}(v_c)$, $\text{GRAPH-GD}(v_b)$ is expected to have more greater distance values on average and we can obtain $\text{GRAPH-GD}(v_b) - \text{GRAPH-GD}(v_c) > 0$.
\end{proof} 

\vspace{1mm}

\begin{figure}[!t]
  \centering
  \begin{subfigure}[t]{0.38\textwidth}    
    \includegraphics[width=\textwidth]{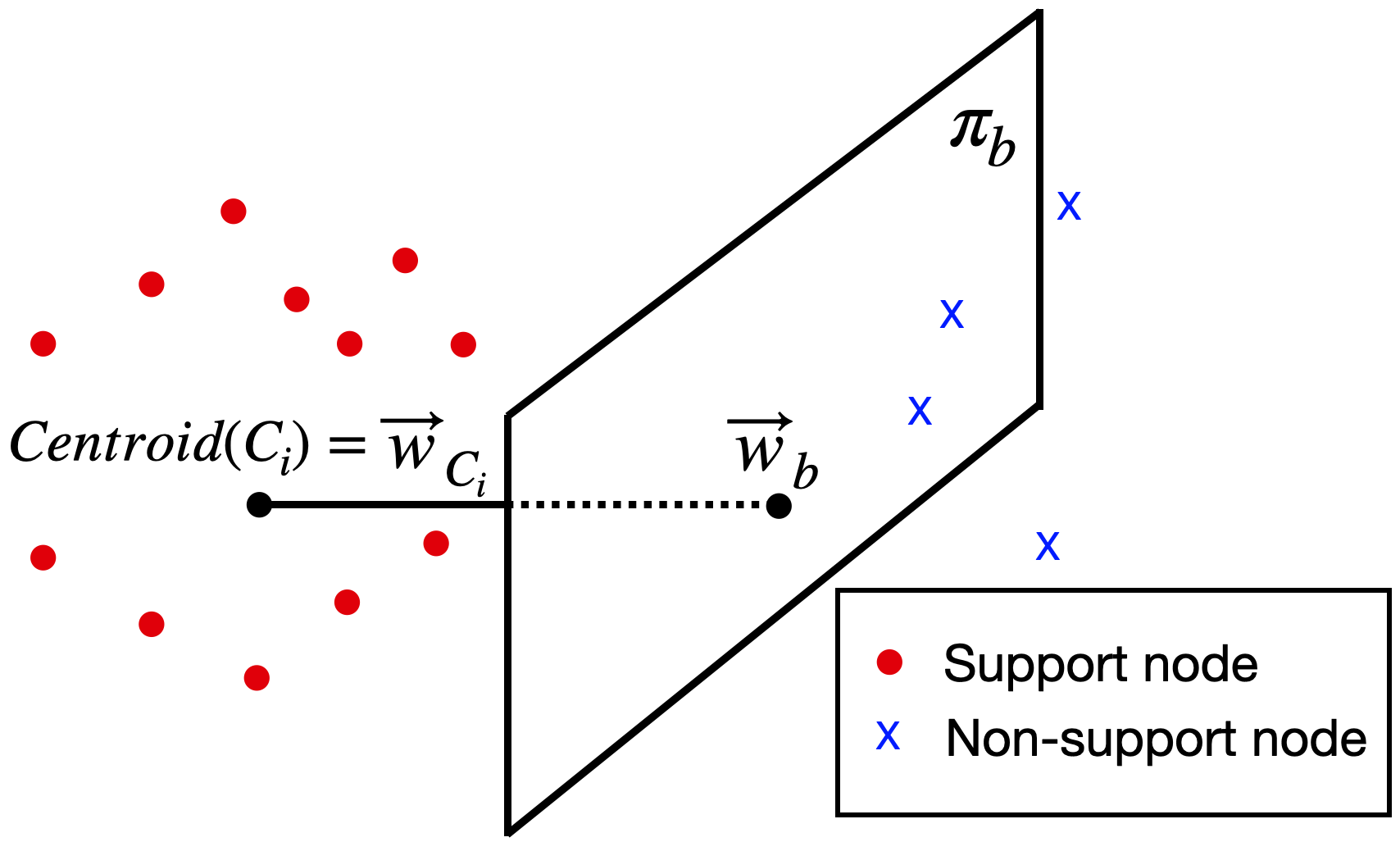}
	  \caption{Support nodes}    
	  \label{subfig:supp:notion}
  \end{subfigure}
  \begin{subfigure}[t]{0.38\textwidth}
    \includegraphics[width=\textwidth]{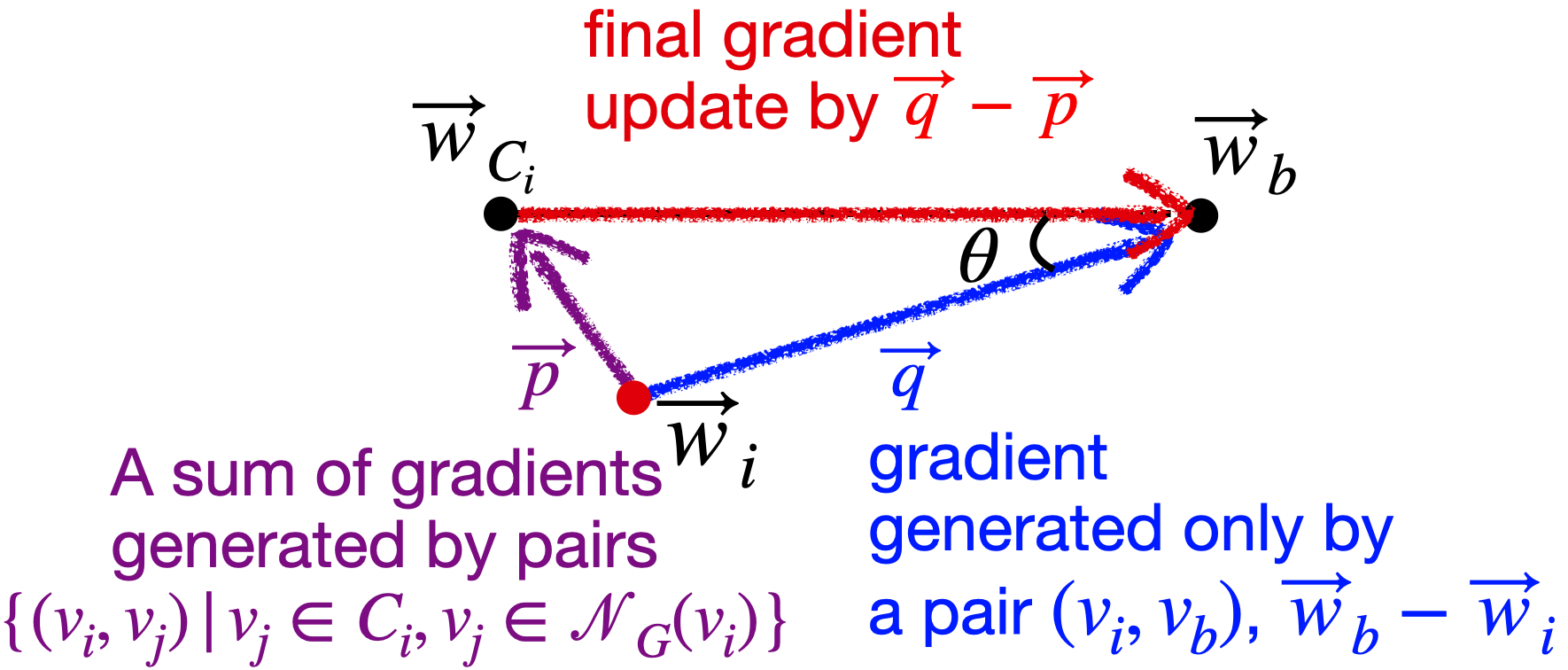}
      \caption{Gradient update in lemma \ref{lem:support:dweight}}    
    \label{subfig:supp:example}
  \end{subfigure}
    \vspace{-2mm}     
  \caption{Illustration of Lemma \ref{lem:support:dweight}. (Best viewed in color.)}    
    \vspace{-6mm}   
\end{figure}

\vspace{-2mm}
Figure \ref{subfig:supp:notion} shows an illustration of \textit{support nodes}, which are colored red.  The support node $v_i$ is placed near $Centroid(C_i)$ owing to the angular distance constraint. We note that the angular distance between $\vec{w}_b - \vec{w}_i$ and $\vec{w}_b-\vec{w}_{C_i}$ is $-90^{\circ}$ and $90^{\circ}$. For conceptual understanding, we can draw a plane (or boundary) $\pi_b$. Here $\pi_b$ is defined by all potential points $\vec{w}_k$ that make $cos(\vec{w}_b - \vec{w}_k, \vec{w}_b-\vec{w}_{C_i}) = 0$ (thus, their degree = $90^{\circ}$ or $-90^{\circ}$). As in Figure \ref{subfig:supp:notion}, $\pi_b$ from the criteria separates between support nodes and non-support nodes.

We expect that  support nodes provide more meaningful evidence to identify the bridge role of $\vec{w}_b$ for $C_i$ than non-support nodes, which could be potentially helpful to enhance GRAPH-GD. We now define the weighting method in GRAPH-wGD, which involves (1) the gradient update of $\vec{w}_i$ with $v_b$ and $v_i$ and (2) their difference vector, $\vec{w}_b-\vec{w}_i$, on the embedding space is used for calculating the weight values. The formal definition of directional weight is as follows.

\begin{definition}[Directional weight] Let $v_b$ be a bridge node and $v_i \in \mathcal{N}_G(v_b)$.  Here 
$\vec{w}_b$ and $\vec{w}_{i}$ are the embeddings of $v_b$ and $v_i$, respectively, and $-\frac { \partial \mathcal{L}_{\mathcal{M}}(v_b, v_i) }{ \partial \vec{w}_i}$ represents the gradient update of $\vec{w}_i$ using $v_b$ and $v_i$. The directional weight function is then defined as $h(v_b, v_i) = 1+ \overline{cos} \left(\vec{w}_b -  \vec{w}_{i},  -\frac { \partial \mathcal{L}_{\mathcal{M}}(v_b, v_i) }{ \partial \vec{w}_i}\right)$, where $\overline{cos}$ returns the cosine similarity if the value is positive and returns 0 otherwise. 
\label{def:dweight}
\end{definition}

% \newpage
\begin{lemma}
Let $v_b$ be a bridge node and $\text{B-Cluster}(v_b)$ be a set of clusters that are bridged by $v_b$. Assume that for any $C_i \!\subset\! \text{B-Cluster}(v_b)$, the sum of gradient updates of $\{(v_j,v_k)|\: \forall v_j, v_k \!\in\! \ C_i\}$ converges to the centroid $w_{C_i}$ (owing to dense edges among nodes in $C_i$). If $v_i$ is a support node for giving $v_b$ the bridging role for $C_i$, then $h(v_b, v_i) > 1$. If $v_i$ is a not support node for $C_i$ and $v_b$, then $h(v_b, v_i) = 1$.
\label{lem:support:dweight}
\end{lemma}

\vspace{-2mm}
\begin{proof} As in Figure \ref{subfig:supp:example}, there exists a sum of gradients generated by pairs, $\{(v_i, v_j)|v_j \in C_i,v_j \in \mathcal{N}_{G}(v_i)\}$, and the gradient update using them becomes $\vec{w}_{C_i} - \vec{w}_{i}$ under the assumption of the dense edges among nodes in $C_i$. In other words, the update makes $\vec{w}_i$ move to the centroid of the cluster, which we call $\vec{p}$. Meanwhile, when the gradient generated only by a pair $(v_i, v_b)$ is $\vec{q}$, the final gradient update becomes $\vec{q} - \vec{p}$, which means that the gradient update from $(v_i, v_b)$ is subtracted by the gradients from all pairs from $v_i$. 
Thus, the final update is $\vec{q} - \vec{p} = \vec{w}_{b} - \vec{w}_{C_i}$. If $v_i$ is a support node, then the angular distance between $\vec{w}_b -  \vec{w}_{i}$ and $\vec{w}_{b} - \vec{w}_{C_i}$ is always between $-90^{\circ}$ and $90^{\circ}$. As a result, its cosine distance is larger than 0 and $h(v_b, v_i) > 1$. Similarly, if $v_i$ is not a support node for $C_i$ and $v_b$, its cosine distance is negative. Therefore, $h(v_b, v_i) = 1$.
\end{proof}

\vspace{-2mm}

\noindent Using Definition \ref{def:support}, Definition \ref{def:dweight}, and Lemma \ref{lem:support:dweight}, we can explain how the weight of GRAPH-wGD works. Again, notation ($W$, $\mathcal{N}_G$, and $ \mathcal{L}_{\mathcal{M}}$) in GRAPH-wGD is dropped for clarity.

\begin{lemma}
Let $v_b$ be a bridge node and $\text{B-Cluster}(v_b)$ be a set of clusters that are bridged by $v_b$.
Assume that, for any $C_i \!\subset\! \text{B-Cluster}(v_b)$, the sum of gradient updates of $\{(v_j,v_k)|\: \forall v_j, v_k \!\in\! \ C_i\}$ converges to the centroid $w_{C_i}$ (due to dense edges among the nodes in $C_i$). If $v_b$ is a bridge node and $\mathcal{N}_G(v_b)$ includes at least one support node, then GRAPH-wGD($v_b$) $>$ GRAPH-GD($v_b$).
\label{lem:support:rel:graphgd:graphwgd}
\end{lemma}

\begin{proof}[Proof of Lemma \ref{lem:support:rel:graphgd:graphwgd}] 
By assuming that $v_b$ is a bridge node and $\mathcal{N}_G(v_b)$ includes at least one support node,
\begin{equation*}
	\resizebox{0.8\hsize}{!}{%
	    $\text{GRAPH-wGD}(v_b) - \text{GRAPH-GD}(v_b)  = \text{MEAN}\left\{ \left|\frac { \partial \mathcal{L}_{\mathcal{M}}(v_b, v_i) }{ \partial \vec{w}_i} \right| \cdot (h(v_b, v_i)  - 1), v_i \in \mathcal{N}_G(v_b)\right\} > 0 $
}
\end{equation*}
from Lemma \ref{lem:support:dweight}.
Therefore, GRAPH-wGD($v_b$) $>$ GRAPH-GD($v_b$).
\end{proof} 

\noindent We can also see how GRAPH-GD compares with GRAPH-wGD with nodes that have no support node. Nodes with no support node mean that their context nodes are placed in the other direction against the centroid of the cluster connected. Thus, the node has no role in bridging with the corresponding cluster.

\begin{lemma}
For any $v_c$, when $\mathcal{N}_G(v_c)$ does not include support nodes, GRAPH-GD($v_c$) $=$ GRAPH-wGD($v_c$).
\label{lem:support:rel:graphgd:graphwgd:nosupport}
\end{lemma}
\vspace{-2mm}

\begin{proof}[Proof of Lemma \ref{lem:support:rel:graphgd:graphwgd:nosupport}]

As in Lemma \ref{lem:support:dweight}, if there is no support node in $\mathcal{N}_G(v_c)$ for any $v_c$, then $h(v_c, v_i)=1$ for all $v_i \in \mathcal{N}_G(v_c)$. Therefore, GRAPH-GD($v_c$) $=$ GRAPH-wGD($v_c$).
\end{proof}

The following theorem shows that the directional weight of GRAPH-wGD helps in finding nodes that have higher \textit{bridgeness} in a larger gap than GRAPH-GD.

\vspace{2mm}
\noindent \textbf{Theorem \ref{theorem:graphwgd:bridge}.}
Let $G$ be a graph with a set of (arbitrary) clusters $C$ and associated embedding $W$. Let $v_b \in C_b$ be a bridge node with at least one support node in $\mathcal{N}_{G}(v_b)$, let $v_c \in C_b$ be a node with the same degree but lower \textit{bridgeness} than $v_b$, and let $\mathcal{N}_{G}(v_c)$ have no support node. Assume that $\text{B-Cluster}(v_b) \supseteq \text{B-Cluster}(v_c)$ and that $W$ is a {\em good embedding} as in Definition \ref{def:goodembedding}.
If, for any $C_i \!\subset\! \text{B-Cluster}(v_b)$, the sum of gradient updates of $\{(v_j,v_k)|\: \forall v_j, v_k \!\in\! \ C_i\}$ converges to the centroid $w_{C_i}$ (owing to dense edges among nodes in $C_i$), then GRAPH-wGD($v_b$) $>$ GRAPH-GD($v_b$) $>$ GRAPH-GD($v_c$) $=$ GRAPH-wGD($v_c$). 

\vspace{-2mm}

\begin{proof}

Using Lemma \ref{lem:support:rel:graphgd:graphwgd} and Lemma \ref{lemma:graphgd:bridge}, we can see that GRAPH-wGD($v_b$) $>$ GRAPH-GD($v_b$) and GRAPH-GD($v_b$) $>$ GRAPH-GD($v_c$), respectively. In addition, Lemma \ref{lem:support:rel:graphgd:graphwgd:nosupport} shows that GRAPH-wGD($v_c$) $=$ GRAPH-wGD($v_c$) under the assumption that context nodes of $v_c$ do not include support nodes. As a result, we can prove that GRAPH-wGD($v_b$) $>$ GRAPH-GD($v_b$) $>$ GRAPH-GD($v_c$) $=$ GRAPH-wGD($v_c$).
\end{proof}

\end{document}